\documentclass{article}
\usepackage{matus}

\def\hpsi{\hat\psi}
\def\Pip{\Pi_\perp}
\def\ncR{\nabla\cR}

\title{Characterizing the implicit bias via a primal-dual analysis}

\author{Ziwei Ji\qquad Matus Telgarsky\\
\tt{\{ziweiji2,mjt\}@illinois.edu}\\
University of Illinois, Urbana-Champaign}

\date{}

\begin{document}

\maketitle

\begin{abstract}
  This paper shows that the implicit bias of gradient descent on linearly
  separable data is exactly characterized by the optimal solution of a dual
  optimization problem given by a smoothed margin, even for general losses.
  This is in contrast to prior results, which are often tailored to
  exponentially-tailed losses.
For the exponential loss specifically, with $n$ training examples and $t$
  gradient descent steps, our dual analysis further allows us to prove an
  $O\del{\ln(n)/\ln(t)}$ convergence rate to the $\ell_2$ maximum margin
  direction, when a constant step size is used.
  This rate is tight in both $n$ and $t$, which has not been presented by prior
  work.
  On the other hand, with a properly chosen but aggressive step size schedule,
  we prove $O(1/t)$ rates for both $\ell_2$ margin maximization and implicit bias,
  whereas prior work (including all first-order methods for the general hard-margin
  linear SVM problem) proved $\widetilde{O}(1/\sqrt{t})$ margin rates, or $O(1/t)$
  margin rates to a suboptimal margin, with an implied (slower) bias rate.
Our key observations include that gradient descent on the primal variable
  naturally induces a mirror descent update on the dual variable, and that the
  dual objective in this setting is smooth enough to give a faster rate.
\end{abstract}

\section{Introduction}

Recent work has shown that in deep learning, the solution found by gradient
descent not only gives low training error, but also has low complexity and
thus generalizes well \citep{zhang_gen,spec}.
This motivates the study of the implicit bias of gradient descent: amongst all
choices with low training error, which is preferred by gradient descent?

This topic has been extensively studied recently: specifically on linear
classifiers, \citet{nati_logistic} show that with linearly separable data and
exponentially-tailed losses (such as the exponential loss and the logistic
loss), gradient descent converges to the $\ell_2$ maximum margin direction.
\citet{min_norm} further characterize the implicit bias in the nonseparable
setting, while \citet{GLSS18} consider generic optimization algorithms such as
steepest descent and adaptive gradient descent.

However, as detailed below, most prior results rely on exponentially-tailed
losses, and do not prove tight rates for a range of step size schedules.
In this work, we focus on linear classifiers and linearly separable data, and
make contributions along all these directions:
\begin{itemize}
  \item We prove that for a broad class of losses that asymptote to $0$,
  including exponentially-tailed losses, polynomially-tailed losses and others,
  the gradient descent iterates grow unboundedly, but their directions (i.e.,
  the normalized gradient descent iterates) converge to some point given by the
  dual optimization problem corresponding to a specific smoothed margin function
  with an $O(1/\sqrt{t})$ rate (cf. \Cref{fact:bias_main}).
Previously, \citet{gd_reg} also handle general losses, and they prove that the
  gradient descent iterates converge to the same direction as regularized
  solutions.
  However, they do not further give a closed-form characterization of the
  implicit bias, and their convergence result is asymptotic.

  \item For the exponential/logistic loss, we can use a much more aggressive
  step size schedule, with which we
  prove an $O(\ln(n)/t)$ rate for $\ell_2$ margin maximization (cf.
  \Cref{fact:margin_t}).
  For the exponential loss, we can also prove an $O(\ln(n)/t)$ rate for convergence
  to the implicit bias (cf. \Cref{fact:tight}).
  Such a step size schedule is also used in AdaBoost
  \citep{freund_schapire_adaboost}; however, it does not always maximize the
  corresponding $\ell_1$ margin, with counterexamples given in
  \citep{rudin2004dynamics}.
  To maximize the $\ell_1$ margin, we need to shrink the step sizes, and prior
  work has shown either an $O(1/t)$ convergence rate to a suboptimal margin
  \citep{mjt_margins}, or an $\widetilde{O}(1/\sqrt{t})$ convergence rate to the
  maximum margin \citep{NLGSS18}.
  Their proof ideas can be applied to generic steepest descent, but cannot
  prove our $O(\ln(n)/t)$ margin maximization rate.

  \item On the other hand, with a constant step size that is more widely used,
  for the exponential loss we prove a tight
  $O\del{\ln(n)/\ln(t)}$ rate for the directional convergence of gradient
  descent to the $\ell_2$ maximum margin direction (cf.
  \Cref{fact:tight}).
  Previously, \citet{nati_logistic} prove an $O\del{1/\ln(t)}$ rate, but the
  dependency on $n$ is not specified; it should be carefully handled since the
  denominator only grows at a rate of $\ln(t)$.
  \citet{min_norm} consider general nonseparable data, but their convergence
  rate for the implicit bias in the separable setting is
  $O\del{\sqrt{\ln(n)/\ln(t)}}$, which is quadratically slower than our rate.
\end{itemize}

All of our results are based on a primal-dual analysis of gradient descent.
One key observation is that gradient descent on the primal variable induces
exactly a mirror descent update on the dual variable.
This perspective has been studied in \citep{boosting_mirror} for boosting /
coordinate descent.
However, they only prove an $\widetilde{O}(1/\sqrt{t})$ dual convergence rate,
while we can further prove an $O(1/t)$ dual rate by exploiting the smoothness of
the dual objective (cf. \Cref{fact:dual_main}).
More surprisingly, our dual analysis further gives rise to a faster primal
convergence guarantee (cf. \Cref{fact:dual_main}), which allows us to prove the
$O(1/t)$ margin maximization and implicit bias rates.

\subsection{Related work}\label{sec:related_work}

Margin maximization and implicit bias are heavily studied in the context of
boosting methods
\citep{boosting_margin,schapire_freund_book_final,shai_singer_weaklearn_linsep}.
Boosting methods are themselves a form of coordinate descent, one whose
convergence is difficult to analyze \citep{schapire_adaboost_convergence_rate}; interestingly, the original proof of AdaBoost's empirical risk convergence also
uses an analysis in the dual \citep{collins_schapire_singer_adaboost_bregman},
though without any rate.
This same dual analysis, and also work by \citet{kivinen_warmuth_bregman}, point
out that AdaBoost, in the dual, performs iterative Bregman projection.

As mentioned above, prior work did not achieve margin maximization rates better than
$O(1/\sqrt{t})$, possibly owing to the nonsmoothness of the problem.
This topic is discussed in \citep[Section 3]{NLGSS18},
where it is stated that the current best rate
is $O(1/\sqrt{t})$ for the general hard-margin linear SVM problem
via first-order methods, that is, not merely restricting to the framework in this present
work, which applies gradient descent to smooth losses which asymptote to zero.

Regarding lower bounds,
\citet{margin_lb} prove that, under a few conditions including
$\epsilon^{-2}=O\del{\min\{n,d\}}$ where $n$ is the number of data examples and
$d$ is the input dimension, to maximize the margin up to an additive error of
$\epsilon$, the optimization algorithm has to read
$\Omega\del{\epsilon^{-2}(n+d)}$ entries of the data matrix.
Due to the required condition, this lower bound is basically
$\Omega\del{(n+d)\min\{n,d\}}=\Omega(nd)$.
On the other hand, in this paper we analyze full-batch gradient descent, which
reads $nd$ entries of the data matrix at each step, and therefore does not
violate this lower bound.
More generally, the lower bound for general nonsmooth optimization is
$1/\sqrt{t}$ \citep[Theorem 3.2.1]{nesterov}, which is arguably one source of difficulty,
as the limiting hard-margin problem is nonsmooth.

The implicit bias of gradient descent has also been studied in more complicated
models, such as deep linear networks and homogeneous networks
\citep{nati_lnn,align,kaifeng_jian_margin,chizat_bach_imp,kernel_rich,dir_align}.

\subsection{Preliminaries and assumptions}

In this paper, $\|\cdot\|$ denotes the $\ell_2$-norm.
The dataset is denoted by $\{(x_i,y_i)\}_{i=1}^n$, where $x_i\in\R^d$ satisfies
$\|x_i\|\le1$, and $y_i\in\{-1,+1\}$.
We assume the data examples are linearly separable.
\begin{assumption}\label{cond:sep}
  There exists $u\in\R^d$ such that $y_i \langle u,x_i\rangle>0$ for all $i$.
\end{assumption}

We consider the unbounded, unregularized empirical risk minimization problem:
\begin{align*}
  \min_{w\in\R^d}\cR(w):=\frac{1}{n}\sum_{i=1}^{n}\ell\del{-y_i\langle w,x_i\rangle}=\frac{1}{n}\sum_{i=1}^{n}\ell\del{\langle w,z_i\rangle},
\end{align*}
where $z_i:=-y_ix_i$, and $\ell:\R\to\R$ denotes an increasing loss function.
We collect $z_i$ into a matrix $Z\in\R^{n\times d}$ whose $i$-th row is
$z_i^\top$.
Moreover, given $\xi\in\R^n$, define
\begin{align*}
  \cL(\xi):=\sum_{i=1}^{n}\ell(\xi_i),\quad\textrm{and}\quad\psi(\xi):=\ell^{-1}\del{\cL(\xi)}.
\end{align*}
Consequently, $\cR(w)=\cL(Zw)/n$.
The function $\psi$ is called a ``generalized sum'' \citep{ineqs} or a
``smoothed margin'' \citep{kaifeng_jian_margin}, and will be crucial in our
analysis.
It ensures that the dual variable is properly normalized; without $\ell^{-1}$
the dual variable would converge to $0$ in the separable case.
We make the following assumption on the loss $\ell$.
\begin{assumption}\label{cond:loss}
  The loss function $\ell$ satisfies:
  \begin{enumerate}
    \item $\ell,\ell',\ell''>0$, and $\lim_{z\to-\infty}\ell(z)=0$;

    \item $z\ell'(z)/\ell(z)$ is increasing on $(-\infty,0)$, and
    $\lim_{z\to-\infty}z\ell'(z)=0$;

    \item for all $b\ge1$, there exists $c>0$ (which can depend on $b$), such
    that for all $a>0$, we have
    $\ell'\del{\ell^{-1}(a)}/\ell'\del{\ell^{-1}(ab)}\ge c$;

    \item the function $\psi$ is convex and $\beta$-smooth with respect to the
    $\ell_\infty$ norm.
  \end{enumerate}
\end{assumption}
Examples satisfying \Cref{cond:loss} include the exponential loss
$\lexp(z):=e^z$, the logistic loss $\llog(z):=\ln(1+e^z)$, and
polynomially-tailed losses (cf. \Cref{fact:loss_examples}).
Note that the smoothness constant $\beta$ will affect the convergence rate, and
although $\beta=1$ for the exponential loss, it can be as large as $n$ for other
losses such as the logistic loss (cf. \Cref{fact:psi_smooth}).
In such settings, we can make a finer analysis which uses the smoothness
constant on sublevel sets: for example, for the logistic loss, $\psi$ is
$2$-smooth on $\cbr{\xi\middle|\psi(\xi)\le0}$ (cf. \Cref{fact:warm_start}).
We still assume global smoothness in \Cref{cond:loss} since it can simplify the
analysis a lot and highlight the key ideas.

\section{A primal-dual convergence analysis}\label{sec:dual}

In this section we prove a primal-dual convergence result (cf.
\Cref{fact:bias_main}), which will be used in the following sections to
characterize a general implicit bias and prove a few fast convergence rates.

We analyze gradient descent on the (primal) risk $\cR(w)$.
Gradient descent starts from some initialization $w_0$, and sets
$w_{t+1}:=w_t-\eta_t\nR(w_t)$ for $t\ge0$.
For each gradient descent iterate $w_t$, let $p_t:=Zw_t$ and
$q_t:=\nabla\psi(p_t)$; we call $q_t$ the corresponding dual variable of $w_t$.
Note that
\begin{align*}
  q_{t,i}=\frac{\ell'(p_{t,i})}{\ell'\del{\psi(p_t)}}=\frac{\ell'\del{\langle w_t,z_i\rangle}}{\ell'\del{\psi(Zw_t)}},
\end{align*}
and thus
\begin{align*}
    w_{t+1}=w_t-\eta_t\nR(w_t)=w_t-\heta_tZ^\top q_t,
\end{align*}
where $\heta_t:=\eta_t\ell'\del{\psi(Zw_t)}/n$, which will be convenient
in our analysis.

The key observation is that the induced update on the dual variable $q_t$ is
actually a mirror descent (more exactly, a dual averaging) update:
\begin{align}\label{eq:md}
  p_{t+1}=p_t-\heta_tZZ^\top q_t=p_t-\heta_t\nabla f(q_t),\quad\textrm{and}\quad q_{t+1}=\nabla\psi(p_{t+1}),
\end{align}
where $f(q):=\enVert{Z^\top q}^2/2$.
Therefore we can use a mirror descent analysis to prove a dual convergence
result, which will later help us characterize the implicit bias for general
losses and improve various rates.

Let $\psi^*$ denote the convex conjugate of $\psi$.
Given $q\in\dom\,\psi^*$, the generalized Bregman distance \citep{gordon_mdp}
between $q$ and $q_t$ is defined as
\begin{align*}
  D_{\psi^*}(q,q_t):=\psi^*(q)-\psi^*(q_t)-\langle p_t,q-q_t\rangle.
\end{align*}
It is a generalization of the Bregman distance to the nondifferentiable setting,
since $\psi^*$ may not be differentiable at $q_t$; instead we just use $p_t$ in
the definition.
Here is our main convergence result.
\begin{theorem}\label{fact:dual_main}
  Under Assumptions \ref{cond:sep} and \ref{cond:loss}, for any $t\ge0$ and
  $q\in\dom\,\psi^*$, if $\heta_t\le1/\beta$, then
  \begin{align*}
    f(q_{t+1})\le f(q_t),\quad\textrm{and}\quad \heta_t\del{f(q_{t+1})-f(q)}\le D_{\psi^*}(q,q_t)-D_{\psi^*}(q,q_{t+1}).
  \end{align*}
  As a result,
  \begin{align*}
    f(q_t)-f(q)\le \frac{D_{\psi^*}(q,q_0)-D_{\psi^*}(q,q_t)}{\sum_{j<t}^{}\heta_j}\le \frac{D_{\psi^*}(q,q_0)}{\sum_{j<t}^{}\heta_j}.
  \end{align*}
  Moreover,
  \begin{align*}
    \psi(p_t)-\psi(p_{t+1})\ge\heta_t\del{f(q_t)+f(q_{t+1})}=\frac{\heta_t}{2}\enVert{Z^\top q_t}^2+\frac{\heta_t}{2}\enVert{Z^\top q_{t+1}}^2,
  \end{align*}
  and thus if $\heta_t$ is nonincreasing, then
  \begin{align*}
    \psi(p_0)-\psi(p_t)\ge \sum_{j<t}^{}\heta_j\enVert{Z^\top q_j}^2-\frac{\heta_0}{2}\enVert{Z^\top q_0}^2.
  \end{align*}
\end{theorem}

Here are some comments on \Cref{fact:dual_main}.
\begin{itemize}
  \item If we let $\heta_t=1/\beta$, then we get an $O(1/t)$ dual convergence
  rate.
  By contrast, \citet{boosting_mirror} consider boosting, and can only handle
  step size $\heta_t\propto1/\sqrt{t+1}$ and give an $\widetilde{O}(1/\sqrt{t})$
  dual rate.
  This is because the dual objective $f(q):=\enVert{Z^\top q}^2/2$ for gradient
  descent is smooth, while for boosting the dual objective is given by
  $\enVert{Z^\top q}_\infty^2/2$, which is nonsmooth.
  In some sense, we can handle a constant $\heta_t$ and prove a faster rate
  because \emph{both the primal objective $\psi$ and the dual objective $f$ are
  smooth}.

  \item Moreover, the primal and dual smoothness allow us to prove a faster
  primal convergence rate for $\psi$: note that in addition to the main term
  $\sum_{j<t}^{}\heta_j\enVert{Z^\top q_j}^2$, \Cref{fact:dual_main} only has a
  bounded error term.
  By contrast, if we use a standard smoothness guarantee and $\heta_t=1/\beta$,
  then the error term can be as large as
  $\sum_{j<t}^{}\heta_j\enVert{Z^\top q_j}^2/2$ (cf.
  \Cref{fact:standard_smooth}).
  In fact, the rate for $\psi$ given by \Cref{fact:dual_main} \emph{only differs
  from the natural lower bound by an additive constant}: after $t$ applications of
  the convexity of $\psi$,
  then $\psi(p_0)-\psi(p_t)\le \sum_{j<t}^{}\heta_j\enVert{Z^\top q_j}^2$.

  While a multiplicative constant factor does not hurt the risk bound too much, it can stop us
  from proving an $O(1/t)$ margin maximization rate with a constant step size
  for the exponential loss (cf. \Cref{sec:refine}).

  \item For the exponential loss (and other exponentially-tailed losses),
  \citet{nati_logistic} prove that $w_t$ converges to the maximum margin
  direction.
  This is called an ``implicit bias'' result since it does not follow from
  classical results such as risk minimization, and requires a nontrivial proof
  tailored to the exponential function.
  By contrast, \Cref{fact:dual_main} explicitly shows that the dual iterates
  minimize the dual objective $f$, and the minimum of $f$ is given exactly by the
  maximum margin (cf. \cref{eq:margin_dual}).
  Moreover, this dual perspective and \Cref{fact:dual_main} can help us
  characterize the implicit bias of a general loss function (cf.
  \Cref{fact:bias_main}).
\end{itemize}

\subsection{Proof of \Cref{fact:dual_main}}

Here are some standard results we need; a proof is given in \Cref{app_sec:dual}
for completeness.
Some refined results are given in \Cref{fact:warm_start}, which use the
smoothness constants over sublevel sets that could be much better.
\begin{lemma}\label[lemma]{fact:standard_smooth}
  We have
  \begin{align*}
    \psi(p_{t+1})-\psi(p_t)\le-\heta_t\enVert{Z^\top q_t}^2+\frac{\beta\heta_t^2}{2}\enVert{Z^\top q_t}^2\qquad\textrm{and}\qquad{}D_{\psi^*}(q_{t+1},q_t)\ge \frac{1}{2\beta}\|q_{t+1}-q_t\|_1^2.
  \end{align*}
\end{lemma}

Next is a standard result for mirror descent; a proof is also included in
\Cref{app_sec:dual}.
\begin{lemma}\label[lemma]{fact:md_step}
  For any $t\ge0$ and $q\in\dom\,\psi^*$, it holds that
  \begin{align*}
    \heta_t\del{f(q_t)-f(q)}\le\ip{\heta_t\nabla f(q_t)}{q_t-q_{t+1}}-D_{\psi^*}(q_{t+1},q_t)+D_{\psi^*}(q,q_t)-D_{\psi^*}(q,q_{t+1}).
  \end{align*}
  Moreover, $q_{t+1}$ is the unique minimizer of
  \begin{align*}
    h(q):=f(q_t)+\ip{\nf(q_t)}{q-q_t}+\frac{1}{\heta_t}D_{\psi^*}(q,q_t),
  \end{align*}
  and specifically $h(q_{t+1})\le h(q_t)=f(q_t)$.
\end{lemma}

On the other hand, we note the $\ell_1$ smoothness of $f$.
\begin{lemma}\label[lemma]{fact:dual_smooth}
  The function $f:\R^n\to\R$ given by $f(q):=\enVert{Z^\top q}^2/2$ is $1$-smooth
  with respect to the $\ell_1$ norm.
\end{lemma}
\begin{proof}
  For any $\theta,\theta'\in\R^n$, using the Cauchy-Schwarz inequality and
  $\|z_i\|\le1$,
  \begin{align*}
    \enVert{\nf(\theta)-\nf(\theta')}_\infty=\enVert{ZZ^\top(\theta-\theta')}_\infty & =\max_{1\le i\le n}\envert{\left \langle Z^\top(\theta-\theta'),z_i\right\rangle} \\
     & \le\max_{1\le i\le n}\enVert{Z^\top(\theta-\theta')}\enVert{z_i} \\
     & \le\enVert{Z^\top(\theta-\theta')}.
  \end{align*}
  Furthermore, by the triangle inequality and $\|z_i\|\le1$,
  \begin{align*}
    \enVert{Z^\top(\theta-\theta')}\le \sum_{i=1}^{n}\envert{\theta_i-\theta'_i}\|z_i\|\le \sum_{i=1}^{n}\envert{\theta_i-\theta'_i}=\enVert{\theta-\theta'}_1.
  \end{align*}
  Therefore $f$ is $1$-smooth with respect to the $\ell_1$ norm.
\end{proof}

With \Cref{fact:standard_smooth,fact:md_step,fact:dual_smooth}, we can prove
\Cref{fact:dual_main}.
\begin{proof}[Proof of \Cref{fact:dual_main}]
  Since $f$ is 1-smooth with respect to the $\ell_1$ norm,
  \begin{align*}
    f(q_{t+1})-f(q_t)\le\ip{\nf(q_t)}{q_{t+1}-q_t}+\frac{1}{2}\|q_{t+1}-q_t\|_1^2.
  \end{align*}
  Further invoking \Cref{fact:standard_smooth}, and $\heta_t\le1/\beta$, and the
  function $h$ defined in \Cref{fact:md_step}, we have
  \begin{align}
        f(q_{t+1}) & \le f(q_t)+\ip{\nf(q_t)}{q_{t+1}-q_t}+\frac{1}{2}\|q_{t+1}-q_t\|_1^2 \nonumber \\
         & \le f(q_t)+\ip{\nf(q_t)}{q_{t+1}-q_t}+\beta D_{\psi^*}(q_{t+1},q_t) \nonumber \\
         & \le f(q_t)+\ip{\nf(q_t)}{q_{t+1}-q_t}+\frac{1}{\heta_t}D_{\psi^*}(q_{t+1},q_t) \label{eq:f_psi_conj} \\
         & =h(q_{t+1})\le f(q_t), \nonumber
  \end{align}
  which proves that $f(q_t)$ is nonincreasing.

  To prove the iteration guarantee for $f$, note that rearranging the terms of
  \cref{eq:f_psi_conj} gives
  \begin{align*}
    \heta_t\ip{\nf(q_t)}{q_{t+1}-q_t}+D_{\psi^*}(q_{t+1},q_t)\ge\heta_t\del{f(q_{t+1})-f(q_t)}.
  \end{align*}
  \Cref{fact:md_step} then implies
  \begin{align*}
    \heta_t\del{f(q_t)-f(q)}\le\heta_t\del{f(q_t)-f(q_{t+1})}+D_{\psi^*}(q,q_t)-D_{\psi^*}(q,q_{t+1}).
  \end{align*}
  Rearranging terms gives
  \begin{align}\label{fact:tmp}
    \heta_t\del{f(q_{t+1})-f(q)}\le D_{\psi^*}(q,q_t)-D_{\psi^*}(q,q_{t+1}).
  \end{align}
  Take the sum of \cref{fact:tmp} from $0$ to $t-1$, we get
  \begin{align*}
      \sum_{j<t}^{}\heta_j\del{f(q_{j+1})-f(q)}\le D_{\psi^*}(q,q_0)-D_{\psi^*}(q,q_t).
  \end{align*}
  Since $f(q_{j+1})\le f(q_j)$ for all $j<t$, we have
  \begin{align*}
      \del{\sum_{j<t}^{}\heta_j}\del{f(q_t)-f(q)}\le D_{\psi^*}(q,q_0)-D_{\psi^*}(q,q_t).
  \end{align*}

  To prove the iteration guarantee for $\psi$, note that
  \begin{align*}
    D_{\psi^*}(q_{t+1},q_t) & =\psi^*(q_{t+1})-\psi^*(q_t)-\langle p_t,q_{t+1}-q_t\rangle \\
     & =\langle p_{t+1},q_{t+1}\rangle-\psi(p_{t+1})-\langle p_t,q_t\rangle+\psi(p_t)-\langle p_t,q_{t+1}-q_t\rangle \\
     & =\psi(p_t)-\psi(p_{t+1})-\langle q_{t+1},p_t-p_{t+1}\rangle \\
     & =\psi(p_t)-\psi(p_{t+1})-\heta_t\ip{Z^\top q_t}{Z^\top q_{t+1}},
  \end{align*}
  where we use \cref{eq:md}.
  Therefore \cref{eq:f_psi_conj} ensures
  \begin{align*}
    \psi(p_t)-\psi(p_{t+1}) & \ge\heta_t\del{f(q_{t+1})-f(q_t)-\ip{\nf(q_t)}{q_{t+1}-q_t}+\ip{Z^\top q_t}{Z^\top q_{t+1}}} \\
     & =\frac{\heta_t}{2}\enVert{Z^\top q_t}^2+\frac{\heta_t}{2}\enVert{Z^\top q_{t+1}}^2.
  \end{align*}
  Telescoping gives the last claim.
\end{proof}

\section{The dual optimal solution characterizes the implicit bias}\label{sec:bias}

As mentioned before, most prior results on the implicit bias are focused on
exponentially-tailed losses.
\citet{gd_reg} consider general losses, and show that the gradient descent
iterates and regularized solutions converge to the same direction, but give no
closed-form characterization of the implicit bias or convergence rate.
In the following result, we characterize the implicit bias using the dual
optimal solution.
\begin{theorem}\label{fact:bias_main}
  Under Assumptions \ref{cond:sep} and \ref{cond:loss}, suppose
  $\heta_t=\eta_t\ell'\del{\psi(Zw_t)}/n\le1/\beta$ is nonincreasing, and
  $\sum_{t=0}^{\infty}\heta_t=\infty$.
  Then for $\barq\in\argmin_{\psi^*(q)\le0}f(q)$, and all $t$ with
  $\psi(Zw_t)\le0$ (which holds for all large enough $t$), we have
  \begin{align*}
    \enVert{Z^\top q_t-Z^\top\barq}^2\le \frac{2D_{\psi^*}(\barq,q_0)}{\sum_{j<t}^{}\heta_j},
    \qquad\textrm{and}
    \qquad
    \ip{\frac{w_t}{\|w_t\|}}{\frac{-Z^\top\barq}{\enVert{Z^\top\barq}}}
    \ge1-\frac{\delta(w_0,\barq)}{\sum_{j<t}^{}\heta_j},
  \end{align*}
  where $\delta(w_0,\barq) := \del{\psi(p_0)+\heta_0f(q_0)+\|w_0\|\enVert{Z^\top\barq}} / \del{2f(\barq)}$
  is a constant depending only on $w_0$ and $\barq$.
  In particular, it holds that the implicit bias is
  $\lim_{t\to\infty}w_t/\|w_t\|=-Z^\top\barq/\enVert{Z^\top\barq}$.
\end{theorem}

\Cref{fact:bias_main} is partly proved by establishing a lower bound on
$-\psi(Zw_t)/\|w_t\|$, which will also help us prove the $O(1/t)$ margin
maximization rate for the exponential loss (cf. \Cref{fact:margin_t}).
Note also that while the condition $\psi^*(q)\leq 0$ in the definition of
$\barq$ looks technical, it appears naturally when deriving $\barq$ as the solution
to the convex dual of the smoothed margin, as in \Cref{app:dual}.

\subsection{Proof of \Cref{fact:bias_main}}

We first prove $Z^\top\barq$ in \Cref{fact:bias_main} is well-defined: there
exists a minimizer $\barq$ of $f(q)$ subject to $\psi^*(q)\le0$, and
$Z^\top\barq$ is unique and nonzero.
\begin{lemma}\label[lemma]{fact:ztq_unique}
  There exists a constant $c>0$, such that for all $q\in\dom\,\psi^*$, it holds
  that $\enVert{Z^\top q}\ge c$.
  In addition, on the sublevel set $S_0:=\cbr{q\middle|\psi^*(q)\le0}$, there
  exists a minimizer $\barq$ of $f$, and $Z^\top\barq$ is the same for all such
  minimizers.
\end{lemma}

The full proof of \Cref{fact:ztq_unique} is given in \Cref{app_sec:bias}.
To show that $\enVert{Z^\top q}$ is bounded below by a positive constant, note
that we only need to consider $q=\nabla\psi(\xi)$ for some $\xi\in\R^n$, and
that \Cref{cond:sep} ensures there exists a unit vector $u\in\R^d$ and
$\gamma>0$ such that for all $1\le i\le n$, it holds that
$\langle u,-z_i\rangle=y_i \langle u,x_i\rangle\ge\gamma$.
Further note that $\nabla\psi(\xi)_i>0$, we have
\begin{align*}
  \enVert{Z^\top\nabla\psi(\xi)}\ge\ip{-Z^\top\nabla\psi(\xi)}{u}=\sum_{i=1}^{n}\langle u,-z_i\rangle\nabla\psi(\xi)_i\ge\gamma\enVert{\nabla\psi(\xi)}_1,
\end{align*}
and the proof is done by showing $\enVert{\nabla\psi(\xi)}_1$ is bounded below by
a positive constant.
To prove the existence of $\barq$, we only need to show $S_0$ is compact.
To prove the uniqueness of $Z^\top\barq$, note that if $\barq_1,\barq_2$ are two
minimizers of $f$ on $S_0$, with
$\enVert{Z^\top\barq_1}=\enVert{Z^\top\barq_2}>0$ but
$Z^\top\barq_1\ne Z^\top\barq_2$, then
\begin{align*}
  \frac{Z^\top\barq_1+Z^\top\barq_2}{2}\in S_0,\quad\textrm{but}\quad\enVert{\frac{Z^\top\barq_1+Z^\top\barq_2}{2}}<\enVert{Z^\top\barq_1},
\end{align*}
a contradiction.

Note that for $\barq$ defined in \Cref{fact:bias_main}, it already follows from
\Cref{fact:dual_main} that $\lim_{t\to\infty}f(q_t)\le f(\barq)$.
We further need the following result to ensure $f(q_t)\ge f(\barq)$; its proof
is basically identical to the proof of \citep[Lemma 3.5]{dir_align}, and is
included in \Cref{app_sec:bias} for completeness.
We then have $\lim_{t\to\infty}=f(\barq)$, which will be used in the proof of
\Cref{fact:bias_main}.
\begin{lemma}\label[lemma]{fact:dual_bound}
  For any $\xi\in\R^n$ such that $\psi(\xi)\le0$, it holds that
  $\psi^*\del{\nabla\psi(\xi)}\le0$.
\end{lemma}

Next we give a formal proof of \Cref{fact:bias_main}.
\begin{proof}[Proof of \Cref{fact:bias_main}]
  Since $\sum_{t=0}^{\infty}\heta_t=\infty$, \Cref{fact:dual_main} implies that
  $\lim_{t\to\infty}f(q_t)\le f(\barq)$.
  On the other hand, by \Cref{fact:dual_main} and \Cref{fact:ztq_unique},
  $\psi(p_t)\le0$ for all large enough $t$, and then \Cref{fact:dual_bound}
  implies that $\psi^*\del{q_t}\le0$.
  By the definition of $\barq$, we have $f(q_t)\ge f(\barq)$, and thus
  $\lim_{t\to\infty}f(q_t)=f(\barq)$.

  We first prove the convergence of $Z^\top q_t$ to $Z^\top\barq$.
  Let $t$ be large enough such that $\psi(p_t)\le0$ and thus $\psi^*(q_t)\le0$.
  Since $\psi^*$ is convex, it follows that the sublevel set
  $\cbr{q\middle|\psi^*(q)\le0}$ is convex, and then the definition of $\barq$
  and the first-order optimality condition
  \citep[Proposition 2.1.1]{borwein_lewis} imply
  \begin{align}\label{eq:ztq_ip}
    \ip{\nf(\barq)}{q_t-q}\ge0,\quad\textup{and thus}\quad\enVert{Z^\top\barq}^2\le\ip{Z^\top\barq}{Z^\top q_t}.
  \end{align}
\Cref{fact:dual_main} and \cref{eq:ztq_ip} then imply
  \begin{align*}
    \enVert{Z^\top q_t-Z^\top\barq}^2 & =\enVert{Z^\top q_t}^2-2\ip{Z^\top q_t}{Z^\top\barq}+\enVert{Z^\top\barq}^2 \\
     & \le\enVert{Z^\top q_t}^2-\enVert{Z^\top\barq}^2 \\
     & \le \frac{2D_{\psi^*}(\barq,q_0)}{\sum_{j<t}^{}\heta_j}.
  \end{align*}

  To prove the other claim, we use an idea from \citep{min_norm}, but also
  invoke the tighter guarantee on $\psi$ in \Cref{fact:dual_main}.
  By Fenchel-Young inequality, and recall that $\psi^*(\barq)\le0$,
  \begin{align}\label{eq:wztq}
    \ip{w_t}{-Z^\top\barq}=-\langle Zw_t,\barq\rangle\ge-\psi(Zw_t)-\psi^*(\barq)\ge-\psi(Zw_t)=-\psi(p_t).
  \end{align}
  Moreover, \Cref{fact:dual_main} implies that
  \begin{align}
    -\psi(p_t) & \ge-\psi(p_0)+\sum_{j<t}^{}\heta_j\enVert{Z^\top q_j}^2-\frac{\heta_0}{2}\enVert{Z^\top q_0}^2 \nonumber \\
     & \ge-\psi(p_0)+\sum_{j<t}^{}\heta_j\enVert{Z^\top q_j}\enVert{Z^\top\barq}-\frac{\heta_0}{2}\enVert{Z^\top q_0}^2. \label{eq:psi_bound}
  \end{align}
  On the other hand, the triangle inequality implies
  $\|w_t\|\le\|w_0\|+\sum_{j<t}^{}\heta_j\enVert{Z^\top q_j}$.
  Recall that $\psi(p_t)\le0$, then \cref{eq:psi_bound} implies
  \begin{align}
    \frac{-\psi(p_t)}{\|w_t\|\enVert{Z^\top\barq}}\ge \frac{-\psi(p_t)}{\del{\|w_0\|+\sum_{j<t}^{}\heta_j\enVert{Z^\top q_j}}\enVert{Z^\top\barq}} & \ge1-\frac{\psi(p_0)+\heta_0f(q_0)+\|w_0\|\enVert{Z^\top\barq}}{\del{\|w_0\|+\sum_{j<t}^{}\heta_j\enVert{Z^\top q_j}}\enVert{Z^\top\barq}} \nonumber \\
     & \ge1-\frac{\psi(p_0)+\heta_0f(q_0)+\|w_0\|\enVert{Z^\top\barq}}{2f(\barq)\sum_{j<t}^{}\heta_j}. \label{eq:-psi}
  \end{align}
  Finally, \cref{eq:wztq,eq:-psi} imply that
  \begin{align*}
    \ip{\frac{w_t}{\|w_t\|}}{\frac{-Z^\top\barq}{\enVert{Z^\top\barq}}} & \ge \frac{-\psi(p_t)}{\|w_t\|\enVert{Z^\top\barq}}\ge1-\frac{\psi(p_0)+\heta_0f(q_0)+\|w_0\|\enVert{Z^\top\barq}}{2f(\barq)\sum_{j<t}^{}\heta_j}
    = 1 - \frac {\delta(w_0,\barq)}{\sum_{j<t}\heta_j}.
  \end{align*}
\end{proof}

\section{$1/t$ and $1/\ln(t)$ exponential loss rates with fast and slow steps}\label{sec:refine}

In this section we focus primarily on the exponential loss $e^z$, proving
refined margin maximization and implicit bias rates, though the margin
maximization rate is also proved for the logistic loss $\ln(1+e^z)$.
We fix the initialization to $w_0=0$, which can make the bounds cleaner; however
our analysis can be easily extended to handle nonzero initialization.

Regarding step sizes, our rates depend on the quantity $\sum_{j<t} \heta_j$,
which at its largest is $t$ by taking constant $\heta_j$, giving rise to both of
our $1/t$ rates.
This step size choice is in fact extremely aggressive; e.g., for the exponential
loss, the induced step sizes
on $\nR(w_j)$ are $\eta_j=1/\cR(w_j)$, which end up growing exponentially.
While this gives our strongest results, for instance improving the known rates
for hard-margin linear SVM as discussed before, such step sizes are rarely used in
practice, and moreover it is unclear if they could carry over to deep learning
and other applications of these ideas, where the step sizes often have constant
or decreasing $\eta_j$.
These smaller step sizes also figure heavily in prior work, and so we give
special consideration to the regime where $\eta_j$ is constant,
and prove a tight $\ln(n)/\ln(t)$ implicit bias rate.

Turning back to additional notation for this section, for the exponential loss,
$\psi(\xi)=\ln\del{\sum_{i=1}^{n}\exp(\xi_i)}$, and
$\psi^*(\theta)=\sum_{i=1}^{n}\theta_i\ln\theta_i$ with domain the standard
probability simplex
$\Delta_n:=\cbr{\theta\in\R^n\middle|\theta\ge0,\sum_{i=1}^{n}\theta_i=1}$.
Moreover, $\psi$ is $1$-smooth with respect to the $\ell_\infty$ norm.

Let $\gamma:=\max_{\|u\|=1}\min_{1\le i\le n}y_i \langle u,x_i\rangle$
and
$\baru:=\argmax_{\|u\|=1}\min_{1\le i\le n}y_i \langle u,x_i\rangle$
respectively denote the maximum margin value and direction
on the dataset.
As in \Cref{app:dual} in the general case of $\psi$, but as presented in prior
work for the specific case of losses with bias towards the maximum margin
solution, the maximum margin has a dual characterization (for the exponential
loss) of
\begin{align}\label{eq:margin_dual}
  \gamma=\min_{q\in\Delta_n}\enVert{Z^\top q}=\sqrt{2\min_{q\in\Delta_n}f(q)}
  = \sup_{\substack{\|w\|\leq 1\\r > 0}} r \psi(Zw/r).
\end{align}

\subsection{$O(1/t)$ margin maximization rates}

For the exponential loss,
\begin{align*}
  \psi(Zw)=\ln\del{\sum_{i=1}^{n}\exp\del{\langle z_i,w\rangle}}\ge\ln\del{\exp\del{\max_{1\le i\le n}\langle z_i,w\rangle}}=\max_{1\le i\le n}\langle z_i,w\rangle,
\end{align*}
and thus
$\min_{1\le i\le n}\langle-z_i,w\rangle=\min_{1\le i\le n}y_i \langle w,x_i\rangle\ge-\psi(Zw)$.
The next result then follows immediately from \cref{eq:psi_bound}, and gives $O(1/t)$
margin maximization rates with constant $\heta_j$.
\begin{theorem}\label{fact:margin_t}
  Under \Cref{cond:sep}, for the exponential loss, if
  $\heta_t=\eta_t\cR(w_t)\le1$ is nonincreasing and $w_0=0$, then
  \begin{align*}
    \frac{\min_{1\le i\le n}y_i\langle w_t,x_i\rangle}{\|w_t\|}\ge \frac{-\psi(Zw_t)}{\|w_t\|}\ge\gamma-\frac{\ln(n)+1}{\gamma\sum_{j<t}^{}\heta_j}.
\end{align*}
For the logistic loss, letting $t_0=(256 \ln n)^2/\gamma^2$, and
  $\eta_t\cR(w_t)=1/2$ for $t<t_0$, and
  $\heta_t = \eta_t\ell'\del{\psi(Zw_t)}/n = 1/2$ for $t\ge t_0$,
  then for any $t > t_0$,
  \[
    \frac{\min_{1\le i\le n}y_i\langle w_t,x_i\rangle}{\|w_t\|}
    \ge
    \frac{-\psi(Zw_t)}{\|w_t\|}
    \ge
    \gamma-\frac{1+512 \ln n}{\gamma t - (256 \ln n)^2 / \gamma}.
  \]
\end{theorem}

The full proof of \Cref{fact:margin_t} is given in \Cref{app_sec:refine}.
Here we sketch the proof for the exponential loss.
Note that \cref{eq:psi_bound,eq:margin_dual} imply
\begin{align*}
  \frac{-\psi(Zw_t)}{\|w_t\|}\ge \frac{-\psi(p_0)+\sum_{j<t}^{}\heta_j\enVert{Z^\top q_j}\cdot\gamma-\frac{\heta_0}{2}\enVert{Z^\top q_0}^2}{\|w_t\|}=\gamma\cdot \frac{\sum_{j<t}^{}\heta_j\enVert{Z^\top q_j}}{\|w_t\|}-\frac{\psi(p_0)+\frac{\heta_0}{2}\enVert{Z^\top q_0}^2}{\|w_t\|}.
\end{align*}
By the triangle inequality,
$\|w_t\|\le \sum_{j<t}^{}\heta_j\enVert{Z^\top q_j}$.
Moreover, $\psi(p_0)=\ln(n)$, and $\enVert{Z^\top q_0}\le1$ since $\|z_i\|\le1$.
Therefore we have
\begin{align*}
  \frac{-\psi(Zw_t)}{\|w_t\|}\ge\gamma-\frac{\ln(n)+1}{\|w_t\|}.
\end{align*}
Lastly, note that $\|w_t\|\ge \langle w_t,\baru\rangle$, and moreover
\begin{align*}
  \langle w_{j+1}-w_j,\baru\rangle=\heta_j\ip{-Z^\top q_j}{\baru}=\heta_j \langle-Z\baru,q_j\rangle\ge\heta_j\gamma.
\end{align*}

Margin maximization has been analyzed in many settings: \citet{mjt_margins}
proves that for any $\epsilon>0$, the margin can be maximized by coordinate
descent to $\gamma-\epsilon$ with an $O(1/t)$ rate, while \citet{NLGSS18} show
an $\widetilde{O}(1/\sqrt{t})$ margin maximization rate for gradient descent by
letting $\heta_t\propto 1/\sqrt{t+1}$ using our notation.
Their proofs also analyze $-\psi(Zw_t)/\|w_t\|$, but use
\Cref{fact:standard_smooth}.
If we let $\heta_t$ be a constant in \Cref{fact:standard_smooth}, the error term
$\sum_{j<t}^{}\frac{\beta\heta_j^2}{2}\enVert{Z^\top q_j}^2$ will be too large
to prove exact margin maximization, while if we let $\heta_t=1/\sqrt{t+1}$, then
the error term is $O\del{\ln(t)}$, but only an $O\del{\ln(t)/\sqrt{t}}$ rate can
be obtained.
By contrast, our analysis uses the tighter guarantee given by
\Cref{fact:dual_main}, which always has a bounded error term.

\subsection{Tight $\ln(n)/t$ and $\ln(n)/\ln(t)$ bias rates}

Next we turn to a fast implicit bias rate, where we produce both upper and lower bounds.
In the case of aggressive step sizes, there does not appear to be prior work,
though a $O(1/t^{1/4})$ rate can be easily derived via the Fenchel-Young inequality from
the bias rate in prior work \citep{NLGSS18}.
Instead, prior work appears to use constant $\eta_j$, and the rate is roughly
$\cO(1/\ln(t))$, however the dependence on $n$ in unspecified, and unclear from the proofs.
Here we provide a careful analysis with a rate $\cO(\ln n / \ln t)$ for
constant $\eta_j$, and rate $\cO(\ln(n)/t)$ for constant $\heta_j$, which we
moreover show are tight.
In this subsection we only analyze the exponential loss.

\begin{theorem}
  \label{fact:tight}
  Consider the exponential loss and nonincreasing steps
  $\heta_t=\eta_t\cR(w_t)\le1$ with $\sum_j \eta_j = \infty$.
  For any data $(z_i)_{i=1}^n$ sampled from a density which is continuous w.r.t. the Lebesgue
  measure and which satisfies \Cref{cond:sep}, then almost surely, for every iteration $t$,
  \[
    \enVert{
      \frac {w_t}{\|w_t\|} - \baru
    }
    =
    \frac{O(\ln n)}{\sum_{j<t} \heta_j}
    =
    \begin{cases}
      O(\frac{\ln n}{t})
      &
      \textrm{when }
      \heta_j = 1,
      \\
      O(\frac{\ln n}{\ln t})
      &
      \textrm{when }
      \eta_j = 1.
    \end{cases}
  \]

  On the other hand, there exists data $Z\in\R^{n\times 2}$ comprised of $n$ examples in $\R^2$
  satisfying \Cref{cond:sep},
  so that for all sufficiently large iterations $t\geq 1$,
  \[
    \enVert{
      \frac {w_t}{\|w_t\|} - \baru
    }
    \geq
    \frac{\ln n - \ln 2}{\|w_t\|}
    =
    \begin{cases}
      \frac {\ln n - \ln 2}{t}
      &
      \textrm{when }
      \heta_j = 1,
      \\
      \frac {\ln n - \ln 2}{\Theta(\ln(t))}
      &
      \textrm{when }
      \eta_j = 1.
    \end{cases}
  \]
\end{theorem}

Before sketching the proof, it's worth mentioning the use of Lebesgue measure in the upper bound.
This assumption ensures the support vectors have reasonable structure and simplifies the
behavior orthogonal to the maximum margin predictor $\baru$;
this assumption originated in prior work on implicit bias \citep{nati_logistic}.

To state the upper and lower bounds more explicitly
and to sketch the proof of \Cref{fact:tight} (full details are in the appendices),
we first introduce some additional notation.
Recall that $\gamma$ denotes the maximum margin, and we let
$\baru:=\argmax_{\|u\|=1}\min_{1\le i\le n}y_i \langle u,x_i\rangle$ denote the
maximum margin direction.
Given any vector $a\in\R^d$, let $\Pi_\perp[a]:=a-\langle a,\baru\rangle\baru$
denote its component orthogonal to $\baru$.
Given a gradient descent iterate $w_t$, let $v_t:=\Pi_\perp[w_t]$.
Given a data point $z_i$, let $z_{i,\perp}:=\Pi_\perp[z_i]$.

Let $S:=\cbr{z_i:\langle\baru,-z_i\rangle=\gamma}$ denote the set of support
vectors, and let
\begin{align*}
    \cR_\gamma(w):=\frac{1}{n}\sum_{z_i\in S}^{}\exp\del{\langle w,z_i\rangle}
\end{align*}
denote the risk induced by support vectors, and
\begin{align*}
    \cR_{>\gamma}(w):=\frac{1}{n}\sum_{z_i\not\in S}^{}\exp\del{\langle w,z_i\rangle}
\end{align*}
denote the risk induced by non-support vectors.
In addition, let $S_\perp:=\cbr{z_{i,\perp}:z_i\in S}$, and
\begin{align*}
    \cR_{\perp}(w):=\frac{1}{n}\sum_{z_i\in S}^{}\exp\del{\langle w,z_{i,\perp}\rangle}=\frac{1}{n}\sum_{z\in S_\perp}^{}\exp\del{\langle w,z\rangle}
\end{align*}
denote the risk induced by components of support vectors orthogonal to $\baru$.
By definition,
$\cR_{\perp}(w)=\cR_{\gamma}(w)\exp\del{\gamma \langle w,\baru\rangle}$.
Lastly, let $\gamma':=\min_{z_i\not\in S}\langle\baru,-z_i\rangle-\gamma$ denote
the margin between support vectors and non-support vectors.
If there is no non-support vector, let $\gamma'=\infty$.

Below is our main result.
\begin{theorem}\label{fact:min_norm_main}
    If the data examples are sampled from a density w.r.t. the Lebesgue measure,
    then almost surely $\cR_{\perp}$ has a unique minimizer $\barv$ over
    $\mathrm{span}(S_\perp)$.
    If all $\heta_j = \eta\cR(w_j)\leq 1$ are nonincreasing, then
    \begin{align*}
        \enVert{v_t-\barv}\le\max\cbr{\enVert{v_0-\barv},2}+\frac{2\ln(n)}{\gamma\gamma'}+2.
    \end{align*}
\end{theorem}

The key potential used in the proof of \Cref{fact:min_norm_main} is
$\enVert{v_t-\barv}^2$.
The change in this potential comes from three parts: (i) a part due to support
vectors, which does not increase this potential; (ii) a part due to non-support
vectors, which is controlled by the dual convergence result
\Cref{fact:dual_main}; (iii) a squared gradient term, which is again
controlled via the dual convergence result \Cref{fact:dual_main}.
The full proof is given in \Cref{app_sec:refine}.

\Cref{fact:min_norm_main} implies that
$\|v_t\|\le\|\barv\|+\|v_0-\barv\|+O\del{\ln(n)}$.
On the other hand, it is proved in the prior work
\citep{nati_logistic,min_norm} that $\|w_t\|=\Theta\del{\ln(t)}$.
Therefore
\begin{align*}
  \enVert{\frac{w_t}{\|w_t\|}-\baru}=\frac{\enVert{w_t-\|w_t\|\baru}}{\|w_t\|}\le \frac{2\|v_t\|}{\|w_t\|}\le O\del{\frac{\ln(n)}{\ln(t)}}.
\end{align*}
Below we further show that this bound is tight: $\enVert{v_t-\barv}$ could be
$\Omega\del{\ln(n)}$ for certain datasets.
\begin{theorem}\label{fact:min_norm_lb}
    Consider the dataset in $\R^2$ where $z_1=(0.1,0)$ and $z_2,\ldots,z_n$ are
    all $(0.2,0.2)$.
    Then $\gamma = 0.1$, and $\barv=(0,0)$, and starting from $w_0=(0,0)$, for large enough $t$, we
    have
    \begin{align*}
        \|v_t-\barv\|=\|v_t\|\ge\ln(n)-\ln(2).
    \end{align*}
\end{theorem}
The proof of \Cref{fact:min_norm_lb} is also given in \Cref{app_sec:refine}.

\section{Examples of losses satisfying \Cref{cond:loss}}\label{sec:loss}

It remains a question what loss functions satisfy \Cref{cond:loss}.
Here are some examples:
\begin{itemize}
  \item The exponential loss $\lexp(z):=e^z$.

  \item The logistic loss $\llog(z):=\ln(1+e^z)$.

  \item The polynomial loss $\ell_{\poly,k}(z)$: on $(-\infty,0]$, it is defined
  as
  \begin{align*}
    \ell_{\poly,k}(z):=\frac{1}{(1-z)^k},\textrm{ and thus }\ell_{\poly,k}'(z)=\frac{k}{(1-z)^{k+1}},\textrm{ and }\ell_{\poly,k}''(z)=\frac{k(k+1)}{(1-z)^{k+2}},
  \end{align*}
  for some $k>0$.
  On $(0,\infty)$, we let $\ell_{\poly,k}'(z):=2k-k(1+z)^{-k-1}$, and therefore
  \begin{align*}
    \ell_{\poly,k}(z)=2kz+\frac{1}{(1+z)^k},\textrm{ and }\ell_{\poly,k}''(z)=\frac{k(k+1)}{(1+z)^{k+2}}.
  \end{align*}
\end{itemize}
\begin{theorem}\label{fact:loss_examples}
  \Cref{cond:loss} is satisfied by $\lexp$, $\llog$, and $\ell_{\poly,k}$ for
  all $k>0$.
\end{theorem}

The tricky thing to verify is the convexity and smoothness of $\psi$.
The following result can help us establish the convexity of $\psi$; it is
basically \citep[Theorem 3.106]{ineqs}, and a proof is included in
\Cref{app_sec:loss} for completeness.
\begin{lemma}\label[lemma]{fact:psi_conv}
  If $\ell'^2/(\ell\ell'')$ is increasing on $(-\infty,\infty)$, then $\psi$ is
  convex.
\end{lemma}

On the smoothness of $\psi$, we have the following global estimate.
\begin{lemma}\label[lemma]{fact:psi_smooth}
  For $\lexp$, the smoothness constant $\beta=1$.
  In general, if $\ell''\le c\ell'$ for some constant $c>0$, then $\beta\le cn$.
\end{lemma}
Note that for $\llog$ and $\ell_{\poly,k}$, the above upper bound on the
smoothness constant is $cn$, which looks bad.
However, in these cases $\ell'$ is bounded above by some universal constant
$c'$; therefore to satisfy the condition
$\heta_t=\eta_t\ell'\del{\psi(Zw_t)}/n\le1/\beta$ in
\Cref{fact:dual_main,fact:bias_main}, it is enough if $\eta_t\le1/(cc')$.
In other words, we can still handle a constant step size for gradient descent on
the empirical risk function $\cR$.
A finer approach is to use the smoothness constant on sublevel sets; we
demonstrate this in the next result for the logistic loss.
\begin{lemma}\label[lemma]{fact:warm_start}
  For the logistic loss, on the sublevel set
  $\cbr{\xi\middle|\psi(\xi)\le0}=\cbr{\xi\middle|\cL(\xi)\le\ell(0)}$, it holds
  that $1\le \enVert{\nabla\psi(\xi)}_1\le2$, and $\psi$ is $2$-smooth with respect
  to the $\ell_\infty$ norm.
  Moreover, if $\cL(Zw_t)\le\ell(0)/(2e^2)$, and
  $\heta_t=\eta_t\ell'\del{\psi(Zw_t)}/n\le1/2$, then
  \begin{align*}
    \psi(Zw_{t+1})-\psi(Zw_t)\le(-\heta_t+\heta_t^2)\enVert{Z^\top q_t}^2,\quad\textrm{and}\quad D_{\psi^*}(q_{t+1},q_t)\ge \frac{1}{4}\|q_{t+1}-q_t\|_1^2.
  \end{align*}
\end{lemma}

To prove \Cref{fact:margin_t}, we can proceed as with the exponential loss
over those iterations $[t_0,t]$ where \Cref{fact:warm_start} is in effect,
achieving the same $O(\ln(n)/t)$ rate over those iterations.
To control the magnitude of $t_0$ and $\|w_{t_0}\|$, we can use a delicate but more standard
analysis, giving $t_0 = O((\ln n)^2/\gamma^2)$ and $\|w_{t_0}\|=O((\ln n)/\gamma)$
(cf. \Cref{fact:warm_start:2}).

The full proofs of results in this section are given in \Cref{app_sec:loss}.

\section{Open problems}\label{sec:open}

One open problem is to extend our results to nonlinear models, such as deep
linear or homogeneous networks.
For example, \citet{chizat_bach_imp} prove that gradient descent can maximize
the margin on a 2-homogeneous network, assuming (a different kind of) dual
convergence.
It is very interesting to see if our analysis can be applied to this setting.

Another open problem is to see whether our analysis can be extended to other
training algorithms, such stochastic gradient descent and accelerated gradient
descent.

\subsection*{Acknowledgements}

The authors thank Maxim Raginsky for pointing them to the concept of generalized
sums \citep{ineqs}, and to Daniel Hsu and Nati Srebro for discussion of lower bounds
and the best known rates for the general hard-margin linear SVM problem.
The authors are grateful for support from the NSF under grant IIS-1750051, and
from NVIDIA under a GPU grant.

\bibliography{bib}
\bibliographystyle{plainnat}

\appendix

\section{Dual objective to the smoothed margin}
\label{app:dual}

This appendix justifies calling $\frac 1 2 \|Z^\T q\|^2$ the dual potential via convex duality,
which also gives another appearance of the constraint $\psi^*(q)\leq 0$.

To start, it seems that ideally we would build a duality around $\psi(Zw)/\|w\|$, however
this is not convex.  Instead, consider
the \emph{persective function} $\hpsi$ of $\psi$ (cf. \citep[Section B.2.2]{HULL}),
a standard notion in convex analysis:
\[
  \hpsi(v,r) :=
  \begin{cases}
    r \psi(v/r)
    & r > 0,
    \\
    \lim_{r\downarrow 0} r \psi(v/r)
    &
    r = 0,
    \\
    \infty
    &
    r < 0.
  \end{cases}
\]
The nonnegative scalar $r$ takes on the role of $\nicefrac {1}{\|w\|}$.

A standard fact from convex analysis is that the perspective of a convex function is
also convex (in joint parameters $(v,r)\in\R^{n+1}$).  We will also need the conjugate
of $\hpsi$, and the fact that $\hpsi\to\max$ for both $\lexp$ and $\llog$.

\begin{lemma}
  \label{fact:hpsi}
  If $\psi$ is closed and convex,
  then $\hpsi$ is convex (as a function over $\R^{n+1}$),
  and has conjugate
  \[
    \hpsi^*((q,b))
    = \begin{cases}
      \infty
      & b  > -\psi^*(q),
      \\
      0
      & b \leq -\psi^*(q).
    \end{cases}
  \]
  Furthermore, for $\ell\in\{\llog,\lexp\}$ and $v\in\R^n$ with $v<0$ (coordinate-wise),
  then $r\mapsto \hpsi(v,r)$ is nondecreasing, and
  $\lim_{r\downarrow 0} \hpsi(v,r) = \max_i v_i$.
\end{lemma}
\begin{proof}
  As mentioned above,
  the perspective function of a closed convex function is also closed and convex
  \citep[Section B.2.2]{HULL}.
  For the conjugate,
  since $\hpsi$ is convex and closed,
  \begin{align*}
    \hpsi^*((q,b))
    &= \sup_{v,r} \ip{v}{q} + br - \hpsi(v,r)
    = \sup_{r > 0} r\del{b + \sup_v\sbr{ \ip{v/r}{q} - \psi(v/r)}}
    = \sup_{r > 0} r\del{ b +  \psi^*(q) }
    \\
    &= \begin{cases}
      \infty
      & b  > -\psi^*(q),
      \\
      0
      & b \leq -\psi^*(q).
    \end{cases}
  \end{align*}

  For the second part, let $v < 0$ and $\ell\in\{\lexp,\llog\}$ be given.
  By \citep[Lemma C.5, after correcting the signs on the losses]{dir_align},
  $\ip{v}{\nabla \psi(v)}\leq \psi(v)$,
  meaning in particular $\ip{v/r}{\nabla\psi(v/r)}\leq \psi(v/r)$ for any $r>0$,
  and
  \[
    \frac {\dif}{\dif r} \hpsi(v,r)
    = \psi(v/r) + r \ip{\psi(v/r)}{- v/ r^2}
    = \psi(v/r) - \ip{\psi(v/r)}{v/r} \geq 0,
  \]
  meaning $r\mapsto\hpsi(v,r)$ is nondecreasing.  It only remains to show
  that $\lim_{r\downarrow 0} \hpsi(v,r) = \min_i v_i$.  For $\lexp$,
  this is a consequence of the standard inequalities
  \begin{align*}
    \hpsi(v,r)
    &= r \ln\sum_i \exp(v_i/r)
    \geq r \ln\max_i \exp(v_i/r)
    \\
    &= \max_i v_i
    \\
    &= r \ln \max_i v_i/r
    \geq r\ln \frac 1 n \sum_i v_i/r
    = \hpsi(v,r) - r \ln n,
  \end{align*}
  and thus $\lim_{r\downarrow 0} \hpsi(v,r) = \max_i v_i$.
  For $\llog$, since $\llog^{-1}(z) = \ln(\exp(z) - 1)$,
  defining $M := \max_i v_i$ for convenience,
  it suffices to note that
  \begin{align*}
    \lim_{r\downarrow 0}
    \hpsi_{\log}(v,r)
    &=
    \lim_{r\downarrow 0}
    r\ln\del[2]{\exp( \sum_i \ln(1+\exp(v_i/r)) - 1)}
    \\
    &=
    \lim_{r\downarrow 0}
    r \ln\del[2]{\prod_i (1+\exp(v_i/r)) - 1}
    \\
    &=
    \lim_{r\downarrow 0}
    r \ln\sum_{\substack{S \subseteq \{1,\ldots,n\}\\|S|\geq 1}} \exp\del[2]{\sum_{i\in S} v_i/r}
    \\
    &=
    M+
    \ln
    \lim_{r\downarrow 0}
    \sbr[3]{\sum_{\substack{S \subseteq \{1,\ldots,n\}\\|S|\geq 1}} \exp\del[2]{\sum_{i\in S} (v_i-M)/r}}^r
    = M.
  \end{align*}
\end{proof}

With $\hpsi$ and $\hpsi^*$ in hand, we can easily form a relevant pair of primal-dual problems.

\begin{theorem}
  \label{fact:hpsi_dual}
  Suppose $\psi$ is closed convex, and $\hpsi(w,r)$ is bounded below over $\|w\|\leq 1$.
  Then
  \[
    \max_{\substack{\|w\|\leq 1\\r\geq 0}} - \hpsi(Zw, r)
    =
    \min_{\substack{q\in\R^n\\\psi^*(q)\leq 0}} \|Z^\T q\|_2.
  \]
\end{theorem}
\begin{remark}
  This form makes the primal and dual explicitly the maximum margin $\gamma$
  for exp-tailed losses.  Alternatively, we could use an SVM form of the objective,
  whereby the dual contains $\|Z^\T q\|_2^2$, and is thus closer to $f$.
\end{remark}
\begin{proof}
  Let $v\in\R^{d+1}$ be a single variable for $(w,r)$,
  and let $\iota$ denote the convex indicator of the set $\{v\in\R^{d+1}: \|v_{1:d}\|\leq 1\}$,
  whereby
  \[
    \iota^*(s) = \sup_{\|v_{1:d}\|\leq 1} \ip{v}{s}
    = \begin{cases}
      \|s\|
      & s_{d+1} = 0,
      \\
      \infty
      & s_{d+1} \neq 0.
    \end{cases}
  \]
  Moreover, let $M\in\R^{(n+1)\times (d+1)}$ denote the matrix which is obtained by adding a row
  and a column to $Z$ which are $0$ except in the common $(n+1,d+1)$-th entry where they are $1$,
  whereby $M (w,r) = (Zw, r)$.
  By Fenchel-Rockafellar duality \citep[Section 31]{rockafellar}, since $\hpsi$ is closed convex
  by \Cref{fact:hpsi},
  \[
    \inf_{\substack{\|w\|\leq 1\\r\geq 0}} \hpsi(Zw, r)
    =
    \inf_{v\in\R^{d+1}} \hpsi(Mv) + \iota(v)
    =
    \max_{s\in\R^{n+1}} -\hpsi^*(-s) - \iota^*(M^\T s).
  \]
  By the earlier form of $\iota^*$ and the construction of $M$,
  we have the constraint $s_{n+1} = (M^\T s)_{d+1} = 0$.  Writing $q\in\R^n$
  for the first $n$ coordinates of $S$ and baking in a $0$ for an $(n+1)$-st coordinate,
  and additionally using the form of $\hpsi^*$ from \Cref{fact:hpsi},
  we have the simpler form
  \[
    \inf_{\substack{\|w\|\leq 1\\r\geq 0}} \hpsi(Zw, r)
    = \max_{q\in\R^n} -\hpsi^*(-(q,0)) - \|Z^\T q\|
    = \max\cbr{ - \|Z^\T q\| : q\in\R^n, 0 \geq \psi^*(-q)}.
  \]
  To finish, we replace $q$ with $-q$ in the dual.
\end{proof}

\section{Omitted proofs from \Cref{sec:dual}}\label{app_sec:dual}

\begin{proof}[Proof of \Cref{fact:standard_smooth}]
  Since $\psi$ is $\beta$-smooth with respect to the $\ell_\infty$ norm,
  \begin{align*}
    \psi(p_{t+1})-\psi(p_t) & \le\ip{\nabla\psi(p_t)}{p_{t+1}-p_t}+\frac{\beta}{2}\|p_{t+1}-p_t\|_\infty^2 \\
     & =\ip{q_t}{-\heta_tZZ^\top q_t}+\frac{\beta\heta_t^2}{2}\enVert{ZZ^\top q_t}_\infty^2 \\
     & =-\heta_t\enVert{Z^\top q_t}^2+\frac{\beta\heta_t^2}{2}\enVert{ZZ^\top q_t}_\infty^2.
  \end{align*}
  Moreover, since $\|z_i\|\le1$,
  \begin{align*}
    \enVert{ZZ^\top q_t}_\infty=\max_{1\le i\le n}\envert{\ip{Z^\top q_t}{z_i}}\le\max_{1\le i\le n}\enVert{Z^\top q_t}\|z_i\|\le\enVert{Z^\top q_t}.
  \end{align*}
  As a result,
  \begin{align*}
    \psi(p_{t+1})-\psi(p_t)\le-\heta_t\enVert{Z^\top q_t}^2+\frac{\beta\heta_t^2}{2}\enVert{Z^\top q_t}^2.
  \end{align*}

  On the second claim, note that since $\psi$ is $\beta$-smooth with respect to
  the $\ell_\infty$ norm, \citep[Lemma 2.19]{shalev_online} implies that
  $\psi^*$ is $(1/\beta)$-strongly convex with respect to the $\ell_1$ norm, and
  in particular $D_{\psi^*}(q_{t+1},q_t)\ge\|q_{t+1}-q_t\|_1^2/(2\beta)$.
\end{proof}

\begin{proof}[Proof of \Cref{fact:md_step}]
  Since $f$ is convex, we have
  \begin{align*}
    \heta_t\del{f(q_t)-f(q)}\le\ip{\heta_t\nabla f(q_t)}{q_t-q}=\ip{\heta_t\nabla f(q_t)}{q_t-q_{t+1}}+\ip{\heta_t\nabla f(q_t)}{q_{t+1}-q}.
  \end{align*}
  Recall that $p_{t+1}=p_t-\heta_tZZ^\top q_t=p_t-\heta_t\nabla f(q_t)$,
  therefore
  \begin{align*}
    \heta_t\del{f(q_t)-f(q)} & \le\ip{\heta_t\nabla f(q_t)}{q_t-q_{t+1}}+\ip{\heta_t\nabla f(q_t)}{q_{t+1}-q} \\
     & =\ip{\heta_t\nabla f(q_t)}{q_t-q_{t+1}}+\ip{p_t-p_{t+1}}{q_{t+1}-q}.
  \end{align*}
  It can be verified by direct expansion that
  \begin{align*}
    \ip{p_t-p_{t+1}}{q_{t+1}-q}=D_{\psi^*}(q,q_t)-D_{\psi^*}(q,q_{t+1})-D_{\psi^*}(q_{t+1},q_t),
  \end{align*}
  and thus
  \begin{align*}
    \heta_t\del{f(q_t)-f(q)}\le\ip{\heta_t\nabla f(q_t)}{q_t-q_{t+1}}+D_{\psi^*}(q,q_t)-D_{\psi^*}(q,q_{t+1})-D_{\psi^*}(q_{t+1},q_t).
  \end{align*}

  On the other claim, let $\partial$ denote subdifferential.
  We have
  \begin{align*}
    \partial h(q)=\cbr{\nf(q_t)}+\frac{1}{\heta_t}\del{\partial\psi^*(q)-\cbr{p_t}}.
  \end{align*}
  Note that $q'\in\argmin h(q)$ if and only if $0\in\partial h(q')$, which is
  equivalent to
  \begin{align*}
    p_t-\heta_t\nf(q_t)=p_{t+1}\in\partial\psi^*(q).
  \end{align*}
  By \citep[Theorem 23.5]{rockafellar}, $p_{t+1}\in\partial\psi^*(q)$ if and
  only if $q=\nabla\psi(p_{t+1})$; in other words, $q_{t+1}$ is the unique
  minimizer of $h$, and specifically $h(q_{t+1})\le h(q_t)=f(q_t)$.
\end{proof}

\section{Omitted proofs from \Cref{sec:bias}}\label{app_sec:bias}

\begin{proof}[Proof of \Cref{fact:ztq_unique}]
  We first prove $\enVert{Z^\top q}\ge c$ for all $q\in\dom\,\psi^*$ and some
  positive constant $c$.
  Note that \citep[Theorem 23.5]{rockafellar} ensures
  $\textup{range}\,\nabla\psi=\dom\,\partial\psi^*$, while
  \citep[Theorem 23.4]{rockafellar} ensures $\dom\,\partial\psi^*$ contains
  the relative interior of $\dom\,\psi^*$; therefore we only need to consider
  $q=\nabla\psi(\xi)$ for some $\xi\in\R^n$.

  Given $\xi\in\R^n$, recall that
  \begin{align}\label{eq:nabla_psi}
    \nabla\psi(\xi)_i=\frac{\ell'(\xi_i)}{\ell'\del{\psi(\xi)}}=\frac{\ell'\del{\ell^{-1}\del{\ell(\xi_i)}}}{\ell'\del{\ell^{-1}\del{\sum_{i=1}^{n}\ell(\xi_i)}}}.
  \end{align}
  Let $s:=\max_{1\le i\le n}\ell(\xi_i)$, then
  $\sum_{i=1}^{n}\ell(\xi_i)\le ns$.
  Since $\ell'>0$, and $\ell'$ and $\ell^{-1}$ are increasing,
  \begin{align*}
    \enVert{\nabla\psi(\xi)}_1\ge \frac{\ell'\del{\ell^{-1}\del{s}}}{\ell'\del{\ell^{-1}\del{\sum_{i=1}^{n}\ell(\xi_i)}}}\ge \frac{\ell'\del{\ell^{-1}(s)}}{\ell'\del{\ell^{-1}(ns)}}.
  \end{align*}
  By \Cref{cond:loss}, there exists a constant $c>0$ such that
  \begin{align*}
    \enVert{\nabla\psi(\xi)}_1\ge \frac{\ell'\del{\ell^{-1}(s)}}{\ell'\del{\ell^{-1}(ns)}}\ge c.
  \end{align*}
  On the other hand, \Cref{cond:sep} ensures that there exists $u\in\R^d$ and
  $\gamma>0$ such that $\langle u,-z_i\rangle\ge\gamma$ for all $1\le i\le n$.
  Therefore
  \begin{align*}
    \enVert{Z^\top\nabla\psi(\xi)}\ge\ip{-Z^\top\nabla\psi(\xi)}{u}=\sum_{i=1}^{n}\langle u,-z_i\rangle\nabla\psi(\xi)_i\ge\gamma\enVert{\nabla\psi(\xi)}_1\ge\gamma c.
  \end{align*}

  Next we prove the existence of $\barq$.
  Since $\ell>0$, we have $\ell(\xi_i)<\sum_{i=1}^{n}\ell(\xi_i)$, and since
  $\ell'$ and $\ell^{-1}$ are increasing, it follows from \cref{eq:nabla_psi}
  that $0<\nabla\psi(\xi)_i\le1$, and $\enVert{\nabla\psi(\xi)}_1\le n$.
  As a result, $\textup{range}\,\nabla\psi$ is bounded, and thus $\dom\,\psi^*$
  is bounded.
  Now consider the sublevel set $S_0:=\cbr{q\middle|\psi^*(q)\le0}$.
  Since $\psi^*$ is closed, \citep[Theorem 7.1]{rockafellar} implies that $S_0$
  is closed.
  Because $S_0\subset\dom\,\psi^*$, it follows that $S_0$ is compact.
  As a result, there exists a minimizer $\barq$ of $f$ over $S_0$.

  Next we prove $Z^\top\barq$ is unique.
  Suppose $\barq_1$ and $\barq_2$ both minimize $f$ over $S_0$, with
  $\enVert{Z^\top\barq_1}=\enVert{Z^\top\barq_2}>0$, but
  $Z^\top\barq_1\ne Z^\top\barq_2$.
  It follows that $Z^\top\barq_1$ and $Z^\top\barq_2$ point to different
  directions, and
  \begin{align*}
    \ip{Z^\top\barq_1}{Z^\top\barq_2}<\enVert{Z^\top\barq_1}^2.
  \end{align*}
  Since $\psi^*$ is convex, $S_0$ is convex, and therefore
  $(\barq_1+\barq_2)/2\in S_0$.
  However, we then have
  \begin{align*}
    \enVert{\frac{Z^\top\barq_1+Z^\top\barq_2}{2}}^2=\frac{1}{2}\enVert{Z^\top\barq_1}^2+\frac{1}{2}\ip{Z^\top\barq_1}{Z^\top\barq_2}<\enVert{Z^\top\barq_1}^2,
  \end{align*}
  a contradiction.
\end{proof}

\begin{proof}[Proof of \Cref{fact:dual_bound}]
  Define $\sigma(s):=\ell'\del{\ell^{-1}(s)}\ell^{-1}(s)$.
  Note that \Cref{cond:loss} implies
  \begin{align*}
    \lim_{s\to0}\sigma(s)=0\quad\textrm{and}\quad\sigma(s)/s\textrm{ is increasing on }\del{0,\ell(0)}.
  \end{align*}
  It then follows that $\sigma$ is super-additive on $\del{0,\ell(0)}$, meaning
  for any $a,b>0$ such that $a+b<\ell(0)$, it holds that
  $\sigma(a+b)\ge\sigma(a)+\sigma(b)$.
  In particular, if $\psi(\xi)\le0$, or equivalently if
  $\sum_{i=1}^{n}\ell(\xi_i)\le\ell(0)$, then
  \begin{align*}
    \sum_{i=1}^{n}\sigma\del{\ell(\xi_i)}-\sigma\del{\sum_{i=1}^{n}\ell(\xi_i)}\le0.
  \end{align*}

  Now note that
  \begin{align*}
    \psi^*\del{\nabla\psi(\xi)}=\ip{\nabla\psi(\xi)}{\xi}-\psi(\xi)=\sum_{i=1}^{n}\frac{\ell'(\xi_i)\xi_i}{\ell'\del{\psi(\xi)}}-\psi(\xi),
  \end{align*}
  and thus
  \begin{align*}
    \ell'\del{\psi(\xi)}\psi^*\del{\nabla\psi(\xi)}=\sum_{i=1}^{n}\ell'(\xi_i)\xi_i-\ell'\del{\psi(\xi)}\psi(\xi)=\sum_{i=1}^{n}\sigma\del{\ell(\xi_i)}-\sigma\del{\sum_{i=1}^{n}\ell(\xi_i)}\le0.
  \end{align*}
  Since $\ell'>0$, it follows that $\psi^*\del{\nabla\psi(\xi)}\le0$.
\end{proof}

\section{Omitted proofs from \Cref{sec:refine}}\label{app_sec:refine}

Before proving \Cref{fact:margin_t}, we first prove a few lemmas that will be
needed.
First, we need to upper and lower bound $\|w_t\|$.

\begin{lemma}\label[lemma]{fact:wt_norm}
  Let $w_0=0$, then for the exponential loss,
  \begin{align*}
    \gamma\sum_{j<t}^{}\heta_j\le\|w_t\|\le \sum_{j<t}^{}\heta_j.
  \end{align*}
  For the logistic loss, we have $\|q_j\|_1\ge1$ and
  $\|w_t\|\ge \gamma\sum_{j<t}^{}\heta_j$.
\end{lemma}
\begin{proof}
  For the exponential loss, since $\|z_i\|\le1$ and $q_j\in\Delta_n$, the
  triangle inequality implies $\enVert{Z^\top q_j}\le1$, and moreover
  \begin{align*}
    \|w_t\|\le \sum_{j<t}^{}\heta_j\enVert{Z^\top q_j}\le \sum_{j<t}^{}\heta_j.
  \end{align*}

  On the other hand, by the definition of the maximum margin $\gamma$ and the
  unit maximum margin solution $\baru$, we have
  $\langle-z_i,\baru\rangle=y_i \langle x_i,\baru\rangle\ge\gamma$ for all $i$.
  Moreover, for the exponential loss, $q_j\in\Delta_n$.
  Therefore
  \begin{align*}
    \langle w_{j+1}-w_j,\baru\rangle=\heta_j\ip{-Z^\top q_j}{\baru}=\heta_j \langle-Z\baru,q_j\rangle\ge\heta_j\gamma.
  \end{align*}
  Since $w_0=0$, the Cauchy-Schwarz inequality implies
  \begin{align*}
    \|w_t\|\ge \langle w_t,\baru\rangle=\sum_{j<t}^{}\langle w_{j+1}-w_j,\baru\rangle\ge\gamma \sum_{j<t}^{}\heta_j.
  \end{align*}

  For the logistic loss, the lower bound proof also works, since
  $\|q_j\|_1\ge1$.
  To see this, note that given $\xi\in\R^n$, we have
  \begin{align*}
    \enVert{\nabla\psi(\xi)}_1=\sum_{i=1}^{n}\frac{\ell'(\xi_i)}{\ell'\del{\psi(\xi)}}=\sum_{i=1}^{n}\frac{\ell'\del{\ell^{-1}\del{\ell(\xi_i)}}}{\ell'\del{\ell^{-1}\del{\sum_{i=1}^{n}\ell(\xi_i)}}}.
  \end{align*}
  Consider the function $\rho(z):=\ell'\del{\ell^{-1}(z)}=1-e^{-z}$.
  It holds that $\rho(0)=0$, and on $[0,\infty)$, we have $\rho$ is subadditive:
  for all $a,b>0$, it holds that $\rho(a+b)\le\rho(a)+\rho(b)$.
  Therefore
  \begin{align*}
    \enVert{\nabla\psi(\xi)}_1=\sum_{i=1}^{n}\frac{\ell'\del{\ell^{-1}\del{\ell(\xi_i)}}}{\ell'\del{\ell^{-1}\del{\sum_{i=1}^{n}\ell(\xi_i)}}}=\frac{\sum_{i=1}^{n}\rho\del{\ell(\xi_i)}}{\rho\del{\sum_{i=1}^{n}\ell(\xi_i)}}\ge1.
  \end{align*}
\end{proof}

With these tools in hand, we turn to the margin rates.

\begin{proof}[Proof of \Cref{fact:margin_t}]
  We first consider the exponential loss.
  Note that \cref{eq:psi_bound,eq:margin_dual} imply
  \begin{align*}
    \frac{-\psi(Zw_t)}{\|w_t\|}\ge \frac{-\psi(p_0)+\sum_{j<t}^{}\heta_j\enVert{Z^\top q_j}\cdot\gamma-\frac{\heta_0}{2}\enVert{Z^\top q_0}^2}{\|w_t\|}=\gamma\cdot \frac{\sum_{j<t}^{}\heta_j\enVert{Z^\top q_j}}{\|w_t\|}-\frac{\psi(p_0)+\frac{\heta_0}{2}\enVert{Z^\top q_0}^2}{\|w_t\|}.
  \end{align*}
  By the triangle inequality,
  $\|w_t\|\le \sum_{j<t}^{}\heta_j\enVert{Z^\top q_j}$.
  Moreover, $\psi(p_0)=\ln(n)$, and $\heta_0\le1$, and
  $\enVert{Z^\top q_0}\le1$ since $\|z_i\|\le1$.
  Therefore
  \begin{align*}
    \frac{-\psi(Zw_t)}{\|w_t\|}\ge\gamma\cdot \frac{\sum_{j<t}^{}\heta_j\enVert{Z^\top q_j}}{\sum_{j<t}^{}\heta_j\enVert{Z^\top q_j}}-\frac{\ln(n)+\frac{1}{2}}{\|w_t\|}\ge\gamma-\frac{\ln(n)+1}{\|w_t\|}.
  \end{align*}
  \Cref{fact:wt_norm} then implies the bound.

  Now consider the logistic loss.
  The analysis is divided into two phases: let $t_0$ denote the first iteration
  where the conditions of \Cref{fact:warm_start} hold, after which we may proceed
  as for the exponential loss.
  To bound $t_0$ and $\|w_{t_0}\|$ and in particular to handle
  the iterations before $t_0$, we apply \Cref{fact:warm_start:2},
  which guarantees $t_0 = O( (\ln n)^2 / \gamma^2 )$ and $\|w_{t_0}\| = O( (\ln n)/\gamma )$.

  Now we can apply \Cref{fact:warm_start}, and start the analysis from $w_{t_0}$
  with smoothness constant $2$.
  All the results in \Cref{sec:dual,sec:bias} still hold, and in particular
  \cref{eq:-psi} ensures for $t>t_0$,
  \begin{align*}
    \frac{-\psi(Zw_t)}{\|w_t\|} & \ge \frac{-\psi(p_{t_0})+\sum_{j=t_0}^{t-1}\heta_j\enVert{Z^\top q_j}\cdot\gamma-\frac{\heta_{t_0}}{2}\enVert{Z^\top q_{t_0}}^2}{\|w_{t_0}\|+\sum_{j=t_0}^{t-1}\heta_j\enVert{Z^\top q_j}} \\
     & =\gamma-\frac{\psi(p_{t_0})+\frac{\heta_{t_0}}{2}\enVert{Z^\top q_{t_0}}^2+\|w_{t_0}\|\gamma}{\|w_{t_0}\|+\sum_{j=t_0}^{t-1}\heta_j\enVert{Z^\top q_j}} \\
     & \ge\gamma-\frac{\psi(p_{t_0})+\frac{\heta_{t_0}}{2}\enVert{Z^\top q_{t_0}}^2+\|w_{t_0}\|\gamma}{\gamma\sum_{j=t_0}^{t-1}\heta_j},
  \end{align*}
  where we use $\enVert{Z^\top q_j}\ge\gamma$, since $\|q_j\|_1\ge1$ as given by
  \Cref{fact:wt_norm}.
  By construction, $\psi(p_{t_0})\le0$, and \Cref{fact:warm_start} implies
  $\|q_{t_0}\|_1\le2$.
  Further letting $\heta_j=1/2$, we get
  \begin{align*}
    \frac{-\psi(Zw_t)}{\|w_t\|}\ge\gamma-\frac{1+2\|w_{t_0}\|\gamma}{\gamma(t-t_0)}.
  \end{align*}
  Plugging in the earlier bounds on $t_0$ and $\|w_{t_0}\|$ from \Cref{fact:warm_start:2}
  gives the final left hand side.  To upper bound the left hand side by the exact margin,
  \Cref{fact:hpsi} suffices.
\end{proof}

We then prove the almost-sure existence of $\barv$.
\begin{proof}[Proof of the first part of \Cref{fact:min_norm_main}]
  Theorem 2.1 of \citep{min_norm} ensures that $S_\perp$ can be decomposed into
  two subsets $B$ and $C$, with the following properties:
  \begin{itemize}
    \item The risk induced by $B$
    \begin{align*}
      \cR_B(w):=\frac{1}{n}\sum_{z\in B}^{}\exp\del{\langle w,z\rangle}
    \end{align*}
    is strongly convex over $\mathrm{span}(B)$.

    \item If $C$ is nonempty, then there exists a vector $\tilde{u}$, such that
    $\langle z,\tilde{u}\rangle=0$ for all $z\in B$, and
    $\langle z,\tilde{u}\rangle\ge\tilde{\gamma}>0$ for all $z\in C$.
  \end{itemize}

  On the other hand, Lemma 12 of \citep{nati_logistic} proves that, almost
  surely there are at most $d$ support vectors, and furthermore the $i$-th
  support vector $z_i$ has a positive dual variable $\theta_i$, such that
  $\sum_{z_i\in S}^{}\theta_iz_i=\gamma\baru$.
  As a result,
  \begin{align*}
    \sum_{z_i\in S}^{}\theta_iz_{i,\perp}=\sum_{z_{i,\perp}\in S_\perp}^{}\theta_iz_{i,\perp}=0.
  \end{align*}
  Note that
  \begin{align*}
    0=\left \langle \sum_{z_{i,\perp}\in S_\perp}^{}\theta_iz_{i,\perp},\tilde{u}\right\rangle=\sum_{z_{i,\perp}\in C}^{}\theta_i \langle z_{i,\perp},\tilde{u}\rangle\ge\tilde{\gamma}\sum_{z_{i,\perp}\in C}^{}\theta_i,
  \end{align*}
  which implies that $C$ is empty, and thus $\cR_\perp$ is strongly convex over
  $\mathrm{span}(S_\perp)$.
  The existence and uniqueness of the minimizer $\barv$ follows from strong
  convexity.
\end{proof}

To prove the second part of \Cref{fact:min_norm_main}, we need the following
iteration guarantee.
Note that it holds for the exponential loss and logistic loss, and will later
be used to provide a better ``warm start'' analysis for the logistic loss (cf.
\Cref{fact:warm_start:2})
\begin{lemma}[\citep{min_norm} Lemma 3.4]\label{fact:exp_smooth}
    Suppose $\ell$ is convex, $\ell'\leq\ell$, and $\ell''\leq\ell$.
    For any $t\ge0$, if $\heta_t=\eta_t\cR(w_t)\le1$, then
    \begin{align*}
        \cR(w_{t+1})\le \cR(w_t)-\eta_t\del{1-\frac{\eta_t\cR(w_t)}{2}}\enVert{\nR(w_t)}^2.
    \end{align*}
\end{lemma}
Note that under the condition of \Cref{fact:min_norm_main} that
$\eta_t\le\min\{1,1/\cR(w_0)\}$, \Cref{fact:exp_smooth} implies that $\cR(w_t)$
is nonincreasing.
Moreover, we have the following bound on $\sum_{j<t}^{}\heta_j$ when $\eta_j$ is
a constant.
\begin{lemma}\label{fact:sum_heta_lb}
    Let $w_0=0$ and $\eta_t=\eta\le1$ for all $t$, then
    \begin{align*}
        \sum_{j<t}^{}\heta_j\ge\ln\del{1+\frac{\eta\gamma^2}{2}t}.
    \end{align*}
\end{lemma}
\begin{proof}
  We first need a risk upper bound. Recall that \Cref{fact:exp_smooth} ensures that for any $j<t$, if $\heta_j=\eta_j\cR(w_j)\le1$, then
  \begin{align}\label{eq:risk_dif_gd}
      \cR(w_{j+1})\le \cR(w_j)-\eta_j\del{1-\frac{\eta_j\cR(w_j)}{2}}\enVert{\nR(w_j)}^2.
  \end{align}
  As a result, if we let $\eta_j=\eta\le1/\cR(w_0)=1$, then $\cR(w_j)$ never increases, and the requirement $\heta_j=\eta_j\cR(w_j)\le1$ of \cref{eq:risk_dif_gd} always holds.

  Dividing both sides of \cref{eq:risk_dif_gd} by $\cR(w_j)\cR(w_{j+1})$ and rearranging terms gives
  \begin{align*}
      \frac{1}{\cR(w_{j+1})}\ge \frac{1}{\cR(w_j)}+\eta\del{1-\frac{\eta\cR(w_j)}{2}}\frac{\enVert{\nR(w_j)}^2}{\cR(w_j)\cR(w_{j+1})}.
  \end{align*}
  Notice that
  \begin{align*}
      \enVert{\nR(w_j)}\ge\envert{\ip{\nR(w_j)}{\baru}}=\envert{\frac{1}{n}\sum_{i=1}^{n}\exp\del{-\langle w_j,z_i\rangle}\langle z_i,\baru\rangle}\ge\gamma\cR(w_j),
  \end{align*}
  and thus
  \begin{align}\label{eq:inv_risk_dif_gd}
      \frac{1}{\cR(w_{j+1})}\ge \frac{1}{\cR(w_j)}+\eta\del{1-\frac{\eta\cR(w_j)}{2}}\gamma^2 \frac{\cR(w_j)}{\cR(w_{j+1})}\ge \frac{1}{\cR(w_j)}+\eta\del{1-\frac{\eta\cR(w_j)}{2}}\gamma^2.
  \end{align}
  Since $\eta\cR(w_j)\le1$, \cref{eq:inv_risk_dif_gd} implies
  \begin{align*}
      \frac{1}{\cR(w_{j+1})}\ge \frac{1}{\cR(w_j)}+\eta\del{1-\frac{\eta\cR(w_j)}{2}}\gamma^2\ge \frac{1}{\cR(w_j)}+\frac{\eta}{2}\gamma^2,
  \end{align*}
  and thus
  \begin{align}\label{eq:risk_rate}
      \cR(w_t)\le1/\del{\frac{1}{\cR(w_0)}+\frac{\eta\gamma^2}{2}t}\le1/\del{1+\frac{\eta\gamma^2}{2}t}.
  \end{align}

  Now we prove the lower bound. Notice that $\ln\cR$ is also convex, since it is the composition of ln-sum-exp and a linear mapping. Therefore the convexity of $\ln\cR$ gives
  \begin{align*}
      \ln\cR(w_{j+1})-\ln\cR(w_j)\ge \langle\nabla\ln\cR(w_j),w_{j+1}-w_j\rangle=-\heta_j\enVert{\nabla\ln\cR(w_j)}^2=-\heta_j\enVert{Z^\top q_j}^2.
  \end{align*}
  The triangle inequality ensures $\enVert{Z^\top q_j}\le \sum_{i=1}^{n}q_{j,i}\|z_i\|\le1$, which implies $\ln\cR(w_{j+1})-\ln\cR(w_j)\ge-\heta_j$, and thus
  \begin{align}\label{eq:sum_heta_tmp}
      \sum_{j<t}^{}\heta_j\ge\ln\cR(w_0)-\ln\cR(w_t).
  \end{align}

  Combining \cref{eq:risk_rate,eq:sum_heta_tmp} gives
  \begin{align*}
      \sum_{j<t}^{}\heta_j\ge\ln\cR(w_0)+\ln\del{1+\frac{\eta\gamma^2}{2}t}=\ln\del{1+\frac{\eta\gamma^2}{2}t}.
  \end{align*}
\end{proof}

Next we prove the refined rate for the implicit bias.
\begin{proof}[Proof of the second part of \Cref{fact:min_norm_main}]
  For technical reasons, we consider a range of steps during which
  $\enVert{v_j-\barv}\ge1$.
  If $\enVert{v_t-\barv}\le1$, then the proof is done.
  Otherwise let $t_{-1}$ denote the last step before $t$ such that
  $\enVert{v_{t_{-1}}-\barv}\le1$; if such a step does not exist, let
  $t_{-1}=-1$.
  Furthermore, let $t_0=t_{-1}+1$.
  Since it always holds that
  \begin{align*}
    \enVert{\eta_j\nR(w_j)} & =\eta_j\enVert{\frac{1}{n}\sum_{i=1}^{n}\exp\del{\langle w_j,z_i\rangle}z_i} \\
     & =\eta_j\cR(w_j)\enVert{\sum_{i=1}^{n}\frac{\exp\del{\langle w_j,z_i\rangle}}{\sum_{i'=1}^{n}\exp\del{\langle w_j,z_{i'}\rangle}}z_i} \\
     & \le\eta_j\cR(w_j)
= \heta_j
     \le1,
  \end{align*}
  we have $\enVert{v_{t_0}-\barv}\le\max\{\enVert{v_0-\barv},2\}$.

  Note that
  \begin{align}\label{eq:mn_step1}
    \enVert{v_{j+1}-\barv}^2 & =\enVert{v_j-\barv-\eta_j\Pi_\perp\sbr{\nR(w_j)}}^2 \nonumber \\
     & =\enVert{v_j-\barv}^2-2\eta_j\left\langle\Pip\nR(w_j),v_j-\barv\right\rangle+\eta_j^2\enVert{\Pip\nR(w_j)}^2\nonumber\\
     & =\enVert{v_j-\barv}^2-2\eta_j\left\langle\nR(w_j),v_j-\barv\right\rangle+\eta_j^2\enVert{\Pip\nR(w_j)}^2,
  \end{align}
  where the middle $\Pip$ could be dropped since $\Pip(v_j - \barv) = v_j -\barv$
  and $\Pip=\Pip^\T$ can be moved across the inner product.
  Continuing, this inner product term in \cref{eq:mn_step1} can be decomposed
  into two parts, for support vectors and non-support vectors respectively:
  \begin{align}\label{eq:mn_ip}
    -\left\langle\nR(w_j),v_j-\barv\right\rangle & =\left\langle\frac{1}{n}\sum_{z_i\in S}^{}\exp\del{\langle w_j,z_i\rangle}z_i,\barv-v_j\right\rangle \nonumber \\
     & \quad+\left\langle\frac{1}{n}\sum_{z_i\not\in S}^{}\exp\del{\langle w_j,z_i\rangle}z_i,\barv-v_j\right\rangle.
  \end{align}
  The support vector part in \cref{eq:mn_ip} is non-positive, due to convexity
  of $\cR_{\perp}$:
  \begin{align}\label{eq:mn_s}
    \left\langle\frac{1}{n}\sum_{z_i\in S}^{}\exp\del{\langle w_j,z_i\rangle}z_i,\barv-v_j\right\rangle & =\left\langle\frac{1}{n}\sum_{z_i\in S}^{}\exp\del{\langle w_j,z_i\rangle}z_{i,\perp},\barv-v_j\right\rangle \nonumber \\
     & =\exp\del{-\gamma \langle w_j,\baru\rangle}\left\langle\frac{1}{n}\sum_{z_i\in S}^{}\exp\del{\langle v_j,z_{i,\perp}\rangle}z_{i,\perp},\barv-v_j\right\rangle \nonumber \\
     & =\exp\del{-\gamma \langle w_j,\baru\rangle}\left\langle\nR_{\perp}(v_j),\barv-v_j\right\rangle \nonumber \\
     & \le\exp\del{-\gamma \langle w_j,\baru\rangle}\del{\cR_{\perp}(\barv)-\cR_{\perp}(v_j)}\le0.
  \end{align}
  The part for non-support vectors in \cref{eq:mn_ip} is bounded using the
  Cauchy-Schwarz inequality:
  \begin{align}\label{eq:mn_ns}
    \left\langle\frac{1}{n}\sum_{z_i\not\in S}^{}\exp\del{\langle w_j,z_i\rangle}z_i,\barv-v_j\right\rangle & \le \frac{1}{n}\sum_{z_i\not\in S}^{}\exp\del{\langle w_j,z_i\rangle}\|z_i\|\|v_j-\barv\| \nonumber \\
     & \le\cR_{>\gamma}(w_j)\|v_j-\barv\|.
  \end{align}
  For $t_0\le j<t$, combining \cref{eq:mn_step1,eq:mn_ip,eq:mn_s,eq:mn_ns}, and
  invoking $\|v_j-\barv\|\ge1$,
  \begin{align*}
    \enVert{v_{j+1}-\barv}^2 & \le\enVert{v_j-\barv}^2+2\eta_j\cR_{>\gamma}(w_j)\|v_j-\barv\|+\eta_j^2\enVert{\Pip\nR(w_j)}^2 \\
     & \le\enVert{v_j-\barv}^2+2\eta_j\cR_{>\gamma}(w_j)\|v_j-\barv\|+\eta_j^2\enVert{\Pip\nR(w_j)}^2\enVert{v_j-\barv} \\
     & \le\del{\enVert{v_j-\barv}+\eta_j\cR_{>\gamma}(w_j)+\frac{\eta_j^2}{2}\enVert{\Pip\nR(w_j)}^2}^2,
  \end{align*}
  and thus
  \begin{align}\label{eq:mn_step2}
    \enVert{v_{j+1}-\barv}\le\enVert{v_j-\barv}+\eta_j\cR_{>\gamma}(w_j)+\frac{\eta_j^2}{2}\enVert{\Pip\nR(w_j)}^2.
  \end{align}

  The middle term with $\cR_{>\gamma}$ is bounded using \Cref{fact:dual_main}.
  First we have
  \begin{align}\label{eq:f_lb}
    \frac{1}{2}\enVert{Z^\top q_j}^2\ge \frac{1}{2}\left \langle-Z^\top q_j,\baru\right\rangle^2 & =\frac{1}{2}\langle-Z\baru,q_j\rangle^2 \nonumber \\
     & \ge \frac{1}{2}\del{\gamma+\gamma'\frac{\cR_{>\gamma}(w_j)}{\cR(w_j)}}^2 \nonumber \\
     & \ge \frac{1}{2}\gamma^2+\gamma\gamma'\frac{\cR_{>\gamma}(w_j)}{\cR(w_j)}.
  \end{align}
  As a result, let $\barq$ denote a minimizer of $f(q)=\enVert{Z^\top q}^2/2$,
  then \Cref{fact:dual_main} and \cref{eq:f_lb} ensure
  \begin{align*}
    D_{\mathrm{KL}}(\barq,q_j)-D_{\mathrm{KL}}(\barq,q_{j+1}) & \ge\heta_j\del{f(q_{j+1})-\frac{1}{2}\gamma^2} \\
     & \ge\eta_j\cR(w_{j+1})\gamma\gamma'\frac{\cR_{>\gamma}(w_{j+1})}{\cR(w_{j+1})} \\
     & =\eta_j\gamma\gamma'\cR_{>\gamma}(w_{j+1}).
  \end{align*}
  Later we will need to evaluate $\sum_j \eta_j \cR_{>\gamma}(w_j)$, which
  by applying the above and telescoping gives
  \begin{align}\label{eq:mn_sns}
    \sum_{j=0}^{\infty}\eta_j\cR_{>\gamma}(w_j)=\eta_j\cR_{>\gamma}(w_0)+\sum_{j=1}^{\infty}\eta_j\cR_{>\gamma}(w_j)\le1+\frac{D_{\mathrm{KL}}(\barq,q_0)}{\gamma\gamma'}\le1+ \frac{\ln(n)}{\gamma\gamma'}.
  \end{align}

  The squared gradient term in \cref{eq:mn_step2} can also
  be bounded using \Cref{fact:dual_main}.  To start, by the definition of $\Pip$
  and since $\baru = - Z^\T \barq / \|Z^\T\barq\|$
  and using the first order condition $\|Z^\T \barq\|^2 \leq \ip{Z^\T\barq}{Z^\T q_t}$
  (which appeared earlier as \cref{eq:ztq_ip}),
  and since $\heta_j\leq 1$ are nonincreasing,
  \begin{align*}
    \eta_j^2 \|\Pip\ncR(w_j)\|^2
    &=
\heta_j^2 \|\nabla\psi(w_j) - \nabla\psi(w_j)^\T \baru \baru\|^2
\\
    &=
    \heta_j^2\del{ \|Z^\T q_j\|^2 - \frac{\ip{Z^\T q_j}{Z^\T\barq}^2}{\|Z^\T \barq\|^2} }
    \\
    &\leq
    \heta_j\heta_{j-1}\del{ \|Z^\T q_j\|^2 - \|Z^\T \barq\|^2 }
\\
&\leq
    \heta_{j}\del{ D_{\psi^*}(\barq,q_{j-1}) - D_{\psi^*}(\barq,q_{j})}
    \leq
    D_{\psi^*}(\barq,q_{j-1}) - D_{\psi^*}(\barq,q_{j}).
  \end{align*}
  As mentioned before, we will need to sum across iterations, which telescopes and gives
  \begin{equation}
    \sum_{j=0}^\infty \eta_j^2 \|\Pip\ncR(w_j)\|^2
    \leq
    \eta_0^2 \|\Pip\ncR(w_0)\|^2
    +
    \sum_{j=1}^\infty \eta_j^2 \|\Pip\ncR(w_j)\|^2
    \leq
    1 + D_{\psi^*}(\barq,q_0)
    \leq
    1+\ln(n).
    \label{eq:mn_sos}
  \end{equation}

  Combining these pieces, applying \cref{eq:mn_step2} recursively
  and then controlling the summations with \cref{eq:mn_sns,eq:mn_sos}
  and using $\max\{\gamma,\gamma'\}\leq 1$ gives
  \begin{align*}
    \|v_t-\barv\|
    \le
    \enVert{v_{t_0}-\barv}
    + \sum_{j=0}^\infty \eta_j \cR_{>\gamma}(w_j)
    + \sum_{j=0}^\infty \frac{\eta_j^2}{2} \|\Pip \cR(w_j)\|^2
    \le
    \enVert{v_{t_0}-\barv}+\frac{2\ln(n)}{\gamma\gamma'}+2,
  \end{align*}
  which finishes the proof.
\end{proof}

Below is the proof of the lower bound on $\|v_t-\barv\|$.
\begin{proof}[Proof of \Cref{fact:min_norm_lb}]
    By construction, the only support vector is $z_1=(0.1,0)$, and
    $z_{1,\perp}=(0,0)$. Therefore $\mathrm{span}(S_\perp)=\mathrm{span}\del{\cbr{(0,0)}}=\cbr{(0,0)}$, $\gamma=0.1$, and
    $\barv=(0,0)$. Moreover,
    \begin{align*}
      \cR_{\gamma}(w)=\frac{1}{n}\exp\del{0.1w_1},\quad \mathrm{and}\quad\cR_{>\gamma}(w)=\frac{n-1}{n}\exp\del{0.2(w_1+w_2)},
    \end{align*}
    and for any $t\ge0$,
    \begin{align}\label{eq:r2_grad}
      \nR(w_t)_1=0.1\cR_{\gamma}(w_t)+0.2\cR_{>\gamma}(w_t),\quad \mathrm{and}\quad\nR(w_t)_2=0.2\cR_{>\gamma}(w_t).
    \end{align}

    Recall that $w_0=0$, and thus \cref{eq:r2_grad} implies that
    $w_{t,1},w_{t,2}\le0$, and $\cR_{\gamma}(w_t)\le1/n$ for all $t$.
    As a result, as long as $\cR(w_t)\ge2/n$, it holds that
    $\cR_{>\gamma}(w_t)\ge\cR_{\gamma}(w_t)$ and
    $\envert{\nR(w_t)_2}\ge\envert{\nR(w_t)_1}/2$.

    Let $\tau$ denote the first step when the risk is less than $2/n$:
    \begin{align*}
      \tau=\min\cbr{t:\cR(w_t)<2/n}.
    \end{align*}
    Since $\envert{\nR(w_t)_2}\ge\envert{\nR(w_t)_1}/2$ for all $t<\tau$, we
    have
    \begin{align*}
      |w_{\tau,2}|\ge|w_{\tau,1}|/2.
    \end{align*}
    On the other hand, since $\|z_i\|\le1/3$, it holds that
    $\cR(w_{\tau})\ge\exp\del{-\|w_\tau\|/3}$, which implies that
    \begin{align*}
      \|w_\tau\|\ge3\ln(n/2).
    \end{align*}
    As a result,
    \begin{align*}
      |w_{\tau,2}|\ge\ln(n/2).
    \end{align*}
\end{proof}

Lastly, we put together the preceding pieces to get the main simplified implicit bias bound.

\begin{proof}[Proof of \Cref{fact:tight}]
  For the upper bound, let $Z$ be given as stated, whereby \Cref{fact:min_norm_main}
  holds, and thus almost surely
  \[
    \enVert{v_t-\barv} = \cO(\ln n),
    \qquad\textrm{whereby }
    \|v_t\|
    = \|\barv\| + \|v_t - \barv\|
    = \cO(\ln n).
  \]
  Next,
  \begin{align*}
    \enVert{
      w_t - \baru \|w_t\|
    }^2
    =
    \enVert[1]{
      \Pi_\perp\del{
        w_t - \baru \|w_t\|
      }
    }^2
    +
    \enVert[1]{
      \del{
        w_t - \baru \|w_t\|
      }^\T \baru \baru
    }^2
    =
    \enVert{
      v_t
    }^2
    +
    \del[1]{ w_t^\T \baru - \|w_t\| }^2.
  \end{align*}
  Since $\|v_t\|=\cO(\ln n)$ whereas $\|w_t\|\to \infty$ via $\sum_j \eta_j = \infty$
  and \Cref{fact:wt_norm},
  then for all sufficiently large $t$, $w_t^\T \baru > \|v_t\|$,
  and thus $\|w_t\| \leq w_t^\T \baru + \|v_t\|$, and
  \[
    \enVert{
      v_t
    }^2
    +
    \del[1]{ w_t^\T \baru - \|w_t\| }^2
    \leq
    2\enVert{v_t}^2.
  \]
  As such, combining these pieces with the inequality $\|w_t\|\geq \gamma \sum_{j<t}\heta_j$
  from \Cref{fact:wt_norm},
  \[
    \enVert{
      \frac {w_t}{\|w_t\|} - \baru
    }
    \leq
    \frac {\sqrt 2 \|v_t\|}{\|w_t\|}
    = \cO\del{\frac {\ln n}{\sum_{j<t}\heta_j}}.
  \]
  For $\heta_j=1$, we have $\sum_{j<t}\heta_j=t$.
  For $\eta_j=1$, we have $\sum_{j<t}\heta_j=\Omega(\ln(t))$ from
  \Cref{fact:sum_heta_lb}.

  For the lower bound, let $Z$ be given by the data in \Cref{fact:min_norm_lb},
  and by the guarantee there,
  \begin{align*}
    \enVert{
      \frac {w_t}{\|w_t\|} - \baru
    }
    &=
    \frac{
      \enVert{
        w_t - \baru \|w_t\|
      }
    }
    {\|w_t\|}
    \geq
    \frac{
      \enVert[1]{
        \Pi_\perp\del{
          w_t - \baru \|w_t\|
        }
      }
    }
    {\|w_t\|}
    =
    \frac{
      \enVert{
        v_t
      }
    }
    {\|w_t\|}
    \geq
    \frac{
      \ln n - \ln 2
    }
    {\|w_t\|}.
  \end{align*}
  The proof is now complete after upper bounding $\|w_t\|$.
  For $\heta_j=1$, by \Cref{fact:wt_norm}, we can just take $\|w_t\| \leq t$.
  For $\eta_j= 1$, \citet[Theorem 3]{nati_logistic} show that
  $\|w_t\|=\Theta(\ln(t))$.
\end{proof}

\section{Omitted proofs from \Cref{sec:loss}}\label{app_sec:loss}

We first prove \Cref{fact:psi_conv,fact:psi_smooth}, which can help us check the
convexity and smoothness of $\psi$ in general.

\begin{proof}[Proof of \Cref{fact:psi_conv}]
  Note that $\nabla\psi(\xi)_i=\ell'(\xi_i)/\ell'\del{\psi(\xi)}$, and
  \begin{align}\label{eq:psi_hessian}
    \nabla^2\psi(\xi)=\diag\del{\frac{\ell''(\xi_1)}{\ell'\del{\psi(\xi)}},\ldots,\frac{\ell''(\xi_n)}{\ell'\del{\psi(\xi)}}}-\frac{\ell''\del{\psi(\xi)}}{\ell'\del{\psi(\xi)}}\nabla\psi(\xi)\nabla\psi(\xi)^\top.
  \end{align}
  We need to show that for any $v\in\R^n$,
  \begin{align}\label{eq:psi_conv_tmp1}
    \sum_{i=1}^{n}\frac{\ell''(\xi_i)}{\ell'\del{\psi(\xi)}}v_i^2\ge \frac{\ell''\del{\psi(\xi)}}{\ell'\del{\psi(\xi)}}\del{\sum_{i=1}^{n}\frac{\ell'(\xi_i)}{\ell'\del{\psi(\xi)}}v_i}^2.
  \end{align}
  Note that by the Cauchy-Schwarz inequality,
  \begin{align*}
    \del{\sum_{i=1}^{n}\frac{\ell'(\xi_i)}{\ell'\del{\psi(\xi)}}v_i}^2\le\del{\sum_{i=1}^{n}\frac{\ell''(\xi_i)}{\ell'\del{\psi(\xi)}}v_i^2}\del{\sum_{i=1}^{n}\frac{\ell'(\xi_i)^2}{\ell''(\xi_i)\ell'\del{\psi(\xi)}}},
  \end{align*}
  and therefore to show \cref{eq:psi_conv_tmp1}, we only need to show that
  \begin{align*}
    \frac{\ell'\del{\psi(\xi)}^2}{\ell''\del{\psi(\xi)}}\ge \sum_{i=1}^{n}\frac{\ell'(\xi_i)^2}{\ell''(\xi_i)},
  \end{align*}
  or
  \begin{align}\label{eq:psi_conv_tmp2}
    \frac{\ell'\del{\psi(\xi)}^2}{\ell''\del{\psi(\xi)}}=\frac{\ell'\del{\ell^{-1}\del{\sum_{i=1}^{n}\ell(\xi_i)}}^2}{\ell''\del{\ell^{-1}\del{\sum_{i=1}^{n}\ell(\xi_i)}}}\ge \sum_{i=1}^{n}\frac{\ell'\del{\ell^{-1}\del{\ell(\xi_i)}}^2}{\ell''\del{\ell^{-1}\del{\ell(\xi_i)}}}.
  \end{align}
  Consider the function $\phi:(0,\infty)\to\R$ given by
  \begin{align*}
    \phi(s):=\frac{\ell'\del{\ell^{-1}(s)}^2}{\ell''\del{\ell^{-1}(s)}}.
  \end{align*}
  Note that $\phi(s)/s=\ell'(z)^2/\del{\ell(z)\ell''(z)}$ for $z=\ell^{-1}(s)$,
  and since $\ell'^2/\ell\ell''$ is increasing, it follows that $\phi(s)/s$ is
  increasing on $(0,\infty)$, and $\lim_{s\to0}\phi(s)=0$.
  In other words, $\phi$ is super-additive, which then implies
  \cref{eq:psi_conv_tmp2}.
\end{proof}

\begin{proof}[Proof of \Cref{fact:psi_smooth}]
Similarly to the proof of \citep[Lemma 14]{shalev2007online}, to check that
  $\psi$ is $\beta$-smooth with respect to the $\ell_\infty$ norm, we only need
  to ensure for any $\xi,v\in\R^n$, it holds that
  $v^\top\nabla^2\psi(\xi)v\le\beta\|v\|_\infty^2$.
  By \cref{eq:psi_hessian}, it is enough if
  \begin{align}\label{eq:smooth_check}
    \sum_{i=1}^{n}\frac{\ell''(\xi_i)}{\ell'\del{\psi(\xi)}}v_i^2\le\beta\max_{1\le i\le n}v_i^2.
  \end{align}

  For $\lexp$,
  \begin{align*}
    \frac{\ell''(\xi_i)}{\ell'\del{\psi(\xi)}}=\frac{e^{\xi_i}}{\sum_{i=1}^{n}e^{\xi_i}},\quad\textrm{and thus}\quad \sum_{i=1}^{n}\frac{\ell''(\xi_i)}{\ell'\del{\psi(\xi)}}v_i^2\le\max_{1\le i\le n}v_i^2.
  \end{align*}

  In general, if $\ell''(z)\le c\ell'(z)$, the since
  $\ell'(\xi_i)\le\ell'\del{\psi(\xi)}$, it holds that
  \begin{align*}
    \sum_{i=1}^{n}\frac{\ell''(\xi_i)}{\ell'\del{\psi(\xi)}}\le \sum_{i=1}^{n}\frac{c\ell'(\xi_i)}{\ell'\del{\psi(\xi)}}\le cn,
  \end{align*}
  and thus we can let $\beta=cn$.
\end{proof}

Next we prove \Cref{fact:loss_examples}.
\begin{proof}[Proof of \Cref{fact:loss_examples}]
  The first two conditions of \Cref{cond:loss} are easy to verify in most cases;
  we only check that $\varphi(z):=z\ell'(z)/\ell(z)$ is increasing on
  $(-\infty,0)$ for the logistic loss.
  We have
  \begin{align*}
    \varphi(z)=\frac{z}{(1+e^{-z})\ln(1+e^z)},\textrm{ and }\varphi'(z)=\frac{(1+e^{-z})\ln(1+e^z)+ze^{-z}\ln(1+e^z)-z}{(1+e^{-z})^2\ln(1+e^z)^2}.
  \end{align*}
  Since $(1+e^{-z})\ln(1+e^z)>0$, and
  \begin{align*}
    ze^{-z}\ln(1+e^z)-z=ze^{-z}\del{\ln(1+e^z)-e^z}>0,
  \end{align*}
  since $z<0$ and $\ln(1+e^z)<e^z$, it follows that $\varphi'(z)>0$.

  On the third requirement of \Cref{cond:loss}, for $\lexp$ we have
  $\ell'\del{\ell^{-1}(s)}=s$, and thus the condition holds with $c=1/b$.
  For $\llog$ we have $\ell'\del{\ell^{-1}(s)}=1-e^{-s}$, and the condition
  holds with $c=1/b$.
  For $\ell_{\poly,k}$, if $a\ge\ell(0)/b$, then
  \begin{align*}
    \frac{\ell'\del{\ell^{-1}(a)}}{\ell'\del{\ell^{-1}(ab)}}\ge \frac{\ell'\del{\ell^{-1}\del{\ell(0)/b}}}{2k},
  \end{align*}
  while if $a\le\ell(0)/b$, then $\ell^{-1}(a)\le\ell^{-1}(ab)\le0$.
  Note that on $(-\infty,0)$,
  \begin{align*}
    \ell'\del{\ell^{-1}(s)}=ks^{(k+1)/k},\textrm{ and thus }\frac{\ell'\del{\ell^{-1}(a)}}{\ell'\del{\ell^{-1}(ab)}}=b^{-(k+1)/k}.
  \end{align*}

  We use \Cref{fact:psi_conv} to verify the convexity of $\psi$.
  For $\lexp$, we have $\ell'^2/(\ell\ell'')=1$.
  For $\llog$, we have $\ell'^2/(\ell\ell'')=e^z/\ln(1+e^z)$, which is increasing.
  For $\ell_{\poly,k}$, on $(-\infty,0]$, we have
  $\ell'^2/(\ell\ell'')=k/(k+1)$.
  On $(0,\infty)$,
  \begin{align*}
    \frac{\ell'^2}{\ell\ell''}=\frac{\del{2k-\frac{k}{(1+z)^{k+1}}}^2}{\del{2kz+\frac{1}{(1+z)^k}}\frac{k(k+1)}{(1+z)^{k+2}}}=\frac{k^2\del{2(1+z)^{k+1}-1}^2}{\del{2kz(1+z)^k+1}k(k+1)},
  \end{align*}
  and thus we only need to show
  \begin{align*}
    \alpha(z):=\frac{\del{2(1+z)^{k+1}-1}^2}{2kz(1+z)^k+1}
  \end{align*}
  is increasing on $(0,\infty)$.
  Note that
  \begin{align*}
    \del{2kz(1+z)^k+1}^2\alpha'(z) & =2\del{2(1+z)^{k+1}-1}\cdot2(k+1)(1+z)^k\cdot\del{2kz(1+z)^k+1} \\
     & \ -\del{2(1+z)^{k+1}-1}^2\del{2k(1+z)^k+2kz\cdot k(1+z)^{k-1}},
  \end{align*}
  and therefore we only need to show that on $(0,\infty)$,
  \begin{align*}
    \kappa(z):=2(k+1)(1+z)\cdot\del{2kz(1+z)^k+1}-
    \del{2(1+z)^{k+1}-1}\del{k(1+z)+k^2z}\ge0.
  \end{align*}
  Rearranging terms gives
  \begin{align*}
    \kappa(z)=2k(1+z)^{k+1}(kz+z-1)+(k^2+3k+2)z+(3k+2).
  \end{align*}
  Note that $\kappa(0)=k+2>0$, and when $z\ge0$,
  \begin{align*}
    \kappa'(z) & =2k(k+1)(1+z)^k(kz+z-1)+2k(1+z)^{k+1}(k+1)+(k^2+3k+2) \\
     & =2k(k+1)(1+z)^k(kz+2z)+(k^2+3k+2)>0.
  \end{align*}
  Therefore $\kappa>0$ on $(0,\infty)$.

  The smoothness of $\psi$ is established by \Cref{fact:psi_smooth}.
\end{proof}

\subsection{Warm start tools for the logistic loss}

If we try to prove fast margin rates for the logistic loss directly from
\Cref{fact:bias_main},  we will pay for the bad initial smoothness of the
corresponding $\psi$, which is $n$, and the rate will be $n/t$.
The smoothness later improves, which is proved as follows.

\begin{proof}[Proof of \Cref{fact:warm_start}]
  We first prove that for all $\xi\in\R^n$ with $\psi(\xi)\le0$, it holds that
  $1\le\enVert{\nabla\psi(\xi)}_1\le2$.
  Given $\xi\in\R^n$, we have
  \begin{align*}
    \enVert{\nabla\psi(\xi)}_1=\sum_{i=1}^{n}\frac{\ell'(\xi_i)}{\ell'\del{\psi(\xi)}}=\sum_{i=1}^{n}\frac{\ell'\del{\ell^{-1}\del{\ell(\xi_i)}}}{\ell'\del{\ell^{-1}\del{\sum_{i=1}^{n}\ell(\xi_i)}}}.
  \end{align*}
  Consider the function $\rho(z):=\ell'\del{\ell^{-1}(z)}=1-e^{-z}$.
  It holds that $\rho(0)=0$, and on $z\in[0,\ell(0)]=[0,\ln(2)]$, we have
  $\rho(z)'\in[1/2,1]$, and $\rho$ is subadditive: for all $a,b>0$ with
  $a+b\le\ln(2)$, it holds that $\rho(a+b)\le\rho(a)+\rho(b)$.
  Now note that since $\psi(\xi)\le0$, we have
  $\sum_{i=1}^{n}\ell(\xi_i)\le\ell(0)$, and thus the subadditivity of $\rho$
  implies
  \begin{align*}
    \enVert{\nabla\psi(\xi)}_1=\sum_{i=1}^{n}\frac{\ell'\del{\ell^{-1}\del{\ell(\xi_i)}}}{\ell'\del{\ell^{-1}\del{\sum_{i=1}^{n}\ell(\xi_i)}}}=\frac{\sum_{i=1}^{n}\rho\del{\ell(\xi_i)}}{\rho\del{\sum_{i=1}^{n}\ell(\xi_i)}}\ge1.
  \end{align*}
  On the other hand, the mean value theorem implies
  \begin{align*}
    \enVert{\nabla\psi(\xi)}_1=\frac{\sum_{i=1}^{n}\rho\del{\ell(\xi_i)}}{\rho\del{\sum_{i=1}^{n}\ell(\xi_i)}}\le \frac{\sum_{i=1}^{n}\ell(\xi_i)}{\frac{1}{2}\sum_{i=1}^{n}\ell(\xi_i)}=2.
  \end{align*}

  Then we show that $\psi$ is $2$-smooth with respect to the $\ell_\infty$ norm
  on the sublevel set $\cbr{\xi\middle|\psi(\xi)\le0}$.
  Recall from \cref{eq:smooth_check} that we only need to check
  \begin{align*}
    \sum_{i=1}^{n}\frac{\ell''(\xi_i)}{\ell'\del{\psi(\xi)}}v_i^2\le2\max_{1\le i\le n}v_i^2.
  \end{align*}
  This is true since $\ell''\le\ell'$, and
  $\sum_{i=1}^{n}\ell'(\xi_i)/\ell'\del{\psi(\xi)}=\enVert{\nabla\psi(\xi)}_1\le2$.

  Next we prove the iteration guarantee on $\psi$.
  Let
  \begin{align*}
    \tilde{\eta}:=\argmax\cbr{0\le\heta\le1\middle|\psi\del{Z(w_t-\heta Z^\top q_t)}\le0},\textrm{ and }\tilde{w}:=w_t-\tilde{\eta}Z^\top q_t.
  \end{align*}
  Since $\cL(Zw_t)<\ell(0)$, we have $\psi(Zw_t)<0$, and thus $\tilde{\eta}>0$.
  We claim that $\tilde{\eta}\ge1/2$.
  If this is not true, then we must have $\psi(\tilde{w})=0$.
  Since $\psi$ is convex, and $\psi(Zw_t)<0$, the line between $Zw_t$
  and $Z\tilde{w}$ are all in the sublevel set $\cbr{\xi\middle|\psi(\xi)\le0}$.
  Using $2$-smoothness of $\psi$, and the same analysis as in
  \Cref{fact:standard_smooth}, we have
  \begin{align}
    \psi(Z\tilde{w})-\psi(Zw_t) & \le\ip{q_t}{Z\tilde{w}-Zw_t}+\|Z\tilde{w}-Zw_t\|_\infty^2 \nonumber \\
     & =-\tilde{\eta}\enVert{Z^\top q_t}^2+\tilde{\eta}^2\enVert{ZZ^\top q_t}_\infty^2 \nonumber \\
     & \le-\tilde{\eta}\enVert{Z^\top q_t}^2+\tilde{\eta}^2\enVert{Z^\top q_t}^2. \label{eq:warm_start_iter}
  \end{align}
  Since $\psi(Zw_t)<0$, and $0<\tilde{\eta}\le1/2$ due to our assumption, and
  $\enVert{Z^\top q_t}>0$ by \Cref{fact:ztq_unique}, we have
  $\psi(Z\tilde{w})<0$, a contradiction.
  As a result, $\tilde{\eta}\ge1/2$, and the iteration guarantee follows from
  \cref{eq:warm_start_iter}.

  Next we prove the strong-convexity-style property for $\psi^*$.
  Let $\xi,\xi'$ satisfy
  \begin{align*}
    \psi(\xi),\psi(\xi')\le\ell^{-1}\del{\frac{\ell(0)}{2e^2}},\quad\textrm{or}\quad\cL(\xi),\cL(\xi')\le \frac{\ell(0)}{2e^2}.
  \end{align*}
  Since $\cL(\xi)=\sum_{i=1}^{n}\ell(\xi_i)$, it follows that for all
  $1\le i\le n$, we have $\ell(\xi_i),\ell(\xi'_i)\le\ell(0)/(2e^2)$, and thus
  $\xi_i,\xi'_i\le0$.
  Note that for all $z\le0$, we have $e^z/2\le\ln(1+e^z)\le e^z$, therefore
  \begin{align*}
    \frac{\ell(z+2)}{\ell(z)}=\frac{\ln(1+e^{z+2})}{\ln(1+e^z)}\le \frac{e^{z+2}}{e^z/2}=2e^2.
  \end{align*}
  Consequently, for all $\tilde{\xi}\in\R^n$ such that
  $\|\xi-\tilde{\xi}\|_\infty\le 2$, it holds that $\cL(\tilde{\xi})\le\ell(0)$,
  and thus $\psi(\tilde{\xi})\le0$.
  Now let $\theta=\nabla\psi(\xi)$, and $\theta'=\nabla\psi(\xi')$, and
  \begin{align*}
    \tilde{\xi}_i:=\xi_i-\frac{\|\theta-\theta'\|_1}{2}\cdot\sgn\del{\theta_i-\theta'_i}.
  \end{align*}
  Recall from the proof of \Cref{fact:dual_main} that
  \begin{align}
    D_{\psi^*}(\theta',\theta) & =\psi(\xi)-\psi(\xi')-\langle\theta',\xi-\xi'\rangle \nonumber \\
     & =\psi(\xi)-\psi(\tilde{\xi})+\psi(\tilde{\xi})-\psi(\xi')-\langle \theta',\xi-\xi'\rangle \nonumber \\
     & =\psi(\xi)-\psi(\tilde{\xi})-\langle\theta',\xi-\tilde{\xi}\rangle+\psi(\tilde{\xi})-\psi(\xi')-\langle\theta',\tilde{\xi}-\xi'\rangle. \label{eq:smooth_check_tmp1}
  \end{align}
  Note that $\|\theta\|_1,\|\theta'\|_1\le2$, therefore
  $\|\theta-\theta'\|_1\le4$.
  It then follows that
  $\|\xi-\tilde{\xi}\|_\infty=\|\theta-\theta'\|_1/2\le2$ and
  $\psi(\xi),\psi(\tilde{\xi})\le0$, and since $\psi$ is $2$-smooth on the
  $0$-sublevel set, we have
  \begin{align}\label{eq:smooth_check_tmp2}
    \psi(\xi)-\psi(\tilde{\xi})\ge \langle\theta,\xi-\tilde{\xi}\rangle-\|\xi-\tilde{\xi}\|_\infty^2=\langle\theta,\xi-\tilde{\xi}\rangle-\frac{\|\theta-\theta'\|_1^2}{4}.
  \end{align}
  Then \cref{eq:smooth_check_tmp1,eq:smooth_check_tmp2} and the convexity of
  $\psi$ imply
  \begin{align*}
    D_{\psi^*}(\theta',\theta) & \ge \langle\theta,\xi-\tilde{\xi}\rangle-\frac{\|\theta-\theta'\|_1^2}{4}-\langle\theta',\xi-\tilde{\xi}\rangle+\psi(\tilde{\xi})-\psi(\xi')-\langle\theta',\tilde{\xi}-\xi'\rangle \\
     & \ge \langle\theta,\xi-\tilde{\xi}\rangle-\frac{\|\theta-\theta'\|_1^2}{4}-\langle\theta',\xi-\tilde{\xi}\rangle \\
     & =\langle\theta-\theta',\xi-\tilde{\xi}\rangle-\frac{\|\theta-\theta'\|_1^2}{4}.
  \end{align*}
  By the construction of $\tilde{\xi}$, we have
  $\langle\theta-\theta',\xi-\tilde{\xi}\rangle=\|\theta-\theta'\|_1^2/2$,
  therefore
  \begin{align*}
    D_{\psi^*}(\theta',\theta)\ge \frac{\|\theta-\theta'\|_1^2}{4}.
  \end{align*}
\end{proof}

The preceding analysis requires $\cL(Zw_t)\leq \ell(0) / (2e^2)$.
We now produce a second analysis to handle those initial iterations leading to this condition.

\begin{lemma}
  \label{fact:warm_start:2}
  Consider the logistic loss $\ln(1+e^z)$, with step size
  $\eta_j = 1 / (2\cR(w_j))$.
  Suppose \Cref{cond:sep} holds,
  and let $\gamma$ and $\baru$ denote the corresponding maximum margin
  value and direction.  Then the first iteration $t$
  with $\cL(Zw_t) \leq \ell(0)/(2e^2)$ satisfies  $\psi(Zw_t) \leq 0$ and
  \[
    t \leq  \del{\frac {256 \ln n}{\gamma}}^2
    \qquad\textrm{and}\qquad
    \|w_{t}\| \leq \frac {256 \ln n}{\gamma}.
  \]
\end{lemma}
\begin{proof}
  Let $t$ denote the first iteration with $\cR(w_t) \leq  1 /(32n)$,
  whereby $\cL(Zw_t) \leq \ln(2) / (2e^2)$ and $\psi(Zw_t)\leq 0$.
  Additionally, define $r:= 128\ln n /\gamma$ and $u := r\baru$;
  Expanding
  the square and invoking \Cref{fact:exp_smooth} and convexity of $\cR$,
  for any $j<t$,
  \begin{align*}
    \enVert{w_{j+1} - u}^2
    &=
    \enVert{w_j - u}^2
    - 2\eta_j \ip{\ncR(w_j)}{w_j - u} + \eta_j^2 \|\ncR(w_j)\|^2
    \\
    &\leq
    \enVert{w_j - u}^2
    + 2\eta_j \del{ \cR(u) - \cR(w_j) }
    + \eta_j^2 \|\ncR(w_j)\|^2.
  \end{align*}
  Applying $\sum_{j<t}$ to both sides and telescoping,
  \begin{equation}
    2\sum_{j<t} \eta_j \cR(w_{j})
    + \|w_t - u\|^2
    - \|w_0 - u\|^2
    \leq
    2\sum_{j<t} \eta_j \cR(u)
    + \sum_{j<t} \eta_j^2 \|\ncR(w_j)\|^2.
    \label{eq:magic:log:1}
  \end{equation}
  The various terms in this expression can be simplified as follows.
  \begin{itemize}
    \item
      Since $\eta_j = \frac{1}{2\cR(w_j)}$, then $2\sum_{j<t} \eta_j \cR(w_j) = t$.

    \item
      Since $w_0=0$,
      \[
        \|w_t - u\|^2
        - \|w_0 - u\|^2
        =
        \|w_t\|^2 - 2\ip{w_t}{u}
        =
        \|w_t\|^2 - 2 r \ip{w_t}{\baru}.
      \]

    \item
      Since $\llog \leq \lexp$, by the choice of $r$,
      \[
        \cR(u) \leq \cR_{\exp}(u) \leq \frac 1 n \sum_i \exp(-\ip{z_i}{\baru} r)
        \leq \frac 1 {128n},
      \]
      and using
      $\cR(w_j) \geq \frac {1}{32n}$
      and $\eta_j = \frac{1}{2\cR(w_j)} \leq 16n$
      gives
      \[
        2 \sum_{j<t} \eta_j \cR(u) \leq 2 \sum_{j<t} \frac {16n}{128n}
        = \frac t 4.
      \]

    \item
      Since $\ell'\leq\ell$,
      \[
        \|\ncR(w_j)\|
        =
        \enVert{\frac 1 n \sum_{i=1}^n \ell'(\ip{z_i}{w_j})z_i}
        \leq
        \frac 1 n \sum_{i=1}^n \ell'(\ip{z_i}{w_j})
        \enVert{z_i}
        \leq
        \frac 1 n \sum_{i=1}^n \ell(\ip{z_i}{w_j})
        =
        \cR(w_j),
      \]
      and since $\eta_j = \frac{1}{2\cR(w_j)}$,
      \[
        \sum_{j<t}\eta_j^2 \|\ncR(w_j)\|^2
        \leq \sum_{j<t} \eta_j^2 \cR(w_j)^2
        = \frac t {4}.
      \]
  \end{itemize}
  Combining these inequalities with \cref{eq:magic:log:1} gives
  \begin{align*}
    t + \|w_t\|^2 - 2r \ip{w_t}{\baru}
    &\leq
    2\sum_{j<t} \eta_j \cR(w_{j})
    + \|w_t - u\|^2
    - \|w_0 - u\|^2
    \\
    &\leq
    2\sum_{j<t} \eta_j \cR(u)
    + \sum_{j<t} \eta_j^2 \|\ncR(w_j)\|^2
    \\
    &\leq
    \frac {t}{4} + \frac {t}{4},
  \end{align*}
  which rearranges to give
  \begin{equation}
    \frac {t}{2} + \|w_t\|^2 - 2r \ip{w_t}{\baru} \leq 0.
    \label{eq:magic:log:2}
  \end{equation}
  Since $t\geq 0$, then by Cauchy-Schwarz, $\|w_t\| \leq 2r$.  Similarly,
  Cauchy-Schwarz grants
  \[
    \frac {t}{2} + \|w_t\|^2 - 2r \ip{w_t}{\baru}
    \geq
    \frac {t}{2} + \|w_t\|^2 - 2r\|w_t\|,
  \]
  which is in fact minimized when $\|w_t\| = 2r$, giving
  \[
    \frac t 2  + \|w_t\|^2 - 2r \ip{w_t}{\baru}
    \geq
    \frac t 2 - 2r^2,
  \]
  which combined with \cref{eq:magic:log:2} implies $t \leq 4r^2$.
\end{proof}

\end{document}